\documentclass[english]{article}

\usepackage[T1]{fontenc}
\usepackage[utf8]{inputenc}
\usepackage{geometry}
\usepackage{amssymb, amsmath, amsthm, mathtools,bm}
\usepackage{graphicx}
\usepackage{subcaption}
\usepackage{enumerate, xcolor}
\usepackage{algorithm, algorithmic}%
\usepackage{xspace}
\usepackage{booktabs}
\usepackage{multirow}
\usepackage{hyperref}
\usepackage{cleveref}

\usepackage{float} 
\usepackage{dsfont}

\hypersetup{colorlinks,linkcolor=red,anchorcolor=blue,citecolor=blue}
\makeatletter
\allowdisplaybreaks

\theoremstyle{plain}

\newtheorem{theorem}{Theorem}%
\newtheorem{lemma}{Lemma}

\newtheorem{corollary}{Corollary}
\newtheorem{assumption}{Assumption}
\newtheorem{definition}{Definition}

\newtheorem{claim}{Claim}
\newtheorem{fact}{Fact}

\@ifpackageloaded{cleveref}{\crefname{theorem}{theorem}{theorem}}{}
\@ifpackageloaded{cleveref}{\crefname{lemma}{lemma}{lemma}}{}
\@ifpackageloaded{cleveref}{\crefname{assumption}{assumption}{assumption}}{}
\@ifpackageloaded{cleveref}{\crefname{corollary}{corollary}{corollary}}{}
\@ifpackageloaded{cleveref}{\crefname{definition}{definition}{definition}}{}

\def\letters{a,b,c,d,e,f,g,h,i,j,k,l,m,n,o,p,q,r,s,t,u,v,w,x,y,z}
\def\Letters{A,B,C,D,E,F,G,H,I,J,K,L,M,N,O,P,Q,R,S,T,U,V,W,X,Y,Z}
\def\greekletters{alpha,beta,gamma,Gamma,lambda,Lambda,delta,Delta,zeta,kappa,phi,psi,nu,eta,varsigma,varphi,xi}
\edef\AllLetters{\letters,\Letters}

\def\DefineShortCuts#1#2#3#4#5{%
    \def\DefineShort##1;{\expandafter \providecommand \csname #2\string##1#3\endcsname {{#4{#5{##1}}}}}
    \@for\next:=#1\do{\expandafter\DefineShort\next;}}

\def\DefineShortCutsGreek#1#2#3#4#5{%
    \def\DefineShort##1;{\expandafter \providecommand \csname #2\string##1#3\endcsname {{#4{#5{\csname ##1\endcsname}}}}}
    \@for\next:=#1\do{\expandafter\DefineShort\next;}}

\DefineShortCuts{\AllLetters}{b}{}{\bm}{}
\DefineShortCuts{\AllLetters}{bb}{}{\overline}{\bm}
\DefineShortCuts{\Letters}{c}{}{\mathcal}{}
\DefineShortCuts{\Letters}{}{b}{\mathbb}{\bm}
\DefineShortCutsGreek{\greekletters}{b}{}{\bm}{}
\DefineShortCutsGreek{\greekletters}{bb}{}{\overline}{\bm}

\newcommand\x\bx
\newcommand\y\by
\newcommand\E\Eb
\newcommand\R\Rb
\newcommand\I\Ib
\renewcommand\P\Pb %

\newcommand{\onet}{\bm{1}_n \otimes}
\newcommand{\ave}{\Big(\frac{1}{n} \bm{1}_n^\top \otimes \bI_d\Big)}
\newcommand{\mean}{( \frac1n \bm 1_n \bm 1_n^\top) \otimes \bI_d}

\def \myalg {{DESTRESS}\xspace}

\newcommand{\KO}{K_{\mathsf{out}}}
\newcommand{\KI}{K_{\mathsf{in}}}
\newcommand{\W}{(\bW \otimes \bI_d)}
\newcommand{\WKO}{(\bW_{\mathsf{out}} \otimes \bI_d)}
\newcommand{\WKI}{(\bW_{\mathsf{in}} \otimes \bI_d)}
\newcommand{\alphai}{\alpha_{\mathsf{in}}}
\newcommand{\alphao}{\alpha_{\mathsf{out}}}
\newcommand{\bGI}{\bG_{\mathsf{in}}}
\newcommand{\bGO}{\bG_{\mathsf{out}}}
\newcommand{\xout}{\x^{\mathsf{output}}}

\definecolor{yjc}{RGB}{255,0,255}
\newcommand{\yc}[1]{\textcolor{yjc}{[Yuejie: #1]}}

\title{DESTRESS: Computation-Optimal and Communication-Efficient Decentralized Nonconvex Finite-Sum Optimization}

\author{
	Boyue Li\thanks{Department of Electrical and Computer Engineering, Carnegie Mellon University, Pittsburgh, PA 15213, USA; Emails:
		\texttt{\{boyuel,yuejiec\}@andrew.cmu.edu}.} \\
		CMU \\ 
		\and
    Zhize Li\thanks{Computer, Electrical and Mathematical Sciences and Engineering Division, King Abdullah University of Science and Technology, Thuwal 23955-6900, Kingdom of Saudi Arabia; Email: \texttt{zhize.li@kaust.edu.sa}.} \\
		KAUST \\
		\and
	Yuejie Chi\footnotemark[1] \\
	CMU
}

\date{October 2021; Revised November 2021}

\begin{document}

\maketitle

\begin{abstract}
 
Emerging applications in multi-agent environments such as internet-of-things, networked sensing, autonomous systems and federated learning, call for decentralized algorithms for finite-sum optimizations that are resource-efficient in terms of both computation and communication. In this paper, we consider the prototypical setting where the agents work collaboratively to minimize the sum of local loss functions by only communicating with their neighbors over a predetermined network topology. 
We develop a new algorithm, called \underline{DE}centralized \underline{ST}ochastic \underline{RE}cur\underline{S}ive gradient method\underline{S} (\myalg) for nonconvex finite-sum optimization, which matches the optimal incremental first-order oracle (IFO) complexity of centralized algorithms for finding first-order stationary points, while maintaining communication efficiency. Detailed theoretical and numerical comparisons corroborate that the resource efficiencies of \myalg improve upon prior decentralized algorithms over a wide range of parameter regimes. 
\myalg leverages several key algorithm design ideas including randomly activated stochastic recursive gradient updates with mini-batches for local computation, gradient tracking with extra mixing (i.e., multiple gossiping rounds) for per-iteration communication, together with careful choices of hyper-parameters and new analysis frameworks to provably achieve a desirable computation-communication trade-off. 
\end{abstract}


\section{Introduction}%
\label{sec:introduction}

The proliferation of multi-agent environments in emerging applications such as internet-of-things (IoT), networked sensing and autonomous systems, together with the necessity of training machine learning models using distributed systems in federated learning, leads to a growing need of developing decentralized algorithms for optimizing finite-sum problems. Specifically, the goal is to minimize the global objective function:
\begin{equation}
    \underset{{\x\,\in\, \mathbb{R}^d}}{\text{minimize}} \quad f(\x) : = \frac{1}{N} \sum_{\bm{z}\in\cM} \ell(\x; \bz),
\end{equation}
where $\bx \in \R^d$ denotes the parameter of interest, $\ell(\x; \bz)$ denotes the sample loss of the sample $\bz$, $\cM$ denotes the entire dataset, and $N = |\cM|$ denotes the number of data samples in the entire dataset. Of particular interest to this paper is the nonconvex setting, where $\ell(\x; \bz)$ is nonconvex with respect to $\x$, due to its ubiquity across problems in machine learning and signal processing, including but not limited to nonlinear estimation, neural network training, and so on. 

In a prototypical decentralized environment, however, each agent only has access to a disjoint subset of the data samples, and aims to work collaboratively to optimize $f(\x)$, by only exchanging information with its neighbors over a predetermined network topology. Assuming the data are distributed equally among all agents,\footnote{It is straightforward to generalize to the unequal splitting case with a proper reweighting.} each agent thus possesses $m  := N / n$ samples, and $f(\x)$ can be rewritten as
\begin{equation*}
    f(\x) =  \frac{1}{n}\sum_{i=1}^n f_i(\x) ,
\end{equation*}
where 
\begin{equation*}%
f_i(\x): = \frac{1}{m} \sum_{\bz \in \cM_i}\ell(\x; \bz)  
\end{equation*}
denotes the local objective function averaged over the local dataset $\cM_i$ at the $i$th agent ($1\leq i\leq n$) and $\cM = \cup_{i=1}^n \cM_i$. The communication pattern of the agents is specified via an undirected graph $\mathcal{G}=(\mathcal{V}, \mathcal{E})$,  where $\mathcal{V}$ denotes the set of all agents, and two agents can exchange information if and only if there is an edge in $\mathcal{E}$ connecting them. Unlike the server/client setting, the decentralized setting, sometimes also called the network setting, does not admit a parameter server to facilitate global information sharing, therefore is much more challenging to understand and delineate the impact of the network graph.

Roughly speaking, in a typical decentralized algorithm, the agents alternate between (1)
communication, which propagates local information and enforces consensus, and (2)
computation, which updates individual parameter estimates and improves convergence using information
received from the neighbors. The resource efficiency of a decentralized algorithm can often be measured in terms of its computation complexity and communication complexity.
For example,  communication can be extremely time-consuming and become the
top priority when the bandwidth is limited. On the other hand, minimizing computation, especially at resource-constrained 
agents (e.g., power-hungry IoT or mobile devices), is also critical to ensure the overall efficiency.
Achieving a desired level of resource efficiency for a
decentralized algorithm often requires careful and delicate trade-offs between computation and communication, as these  objectives are often conflicting in nature.

\subsection{Our contributions}
\label{sub:contributions}

The central contribution of this paper lies in the development of a new resource-efficient algorithm for nonconvex finite-sum optimization problems in a decentralized environment, dubbed \underline{DE}centralized \underline{ST}ochastic \underline{RE}cur\underline{S}ive gradient method\underline{S} (\myalg). \myalg provably finds first-order stationary points of the global objective function $f(\x)$ with the optimal incremental first-order (IFO) oracle complexity, i.e. the complexity of evaluating sample gradients, matching state-of-the-art centralized algorithms, but at a much lower communication complexity compared to existing decentralized algorithms over a wide range of parameter regimes. 

To achieve resource efficiency, \myalg leverages several key ideas in the algorithm design. To reduce local computation, \myalg harnesses the finite-sum structure of the empirical risk function by performing stochastic variance-reduced recursive gradient updates \cite{nguyen2019finite,fang2018spider,wang2019spiderboost,li2019ssrgd,li2021zerosarah,li2021page,zhou2020stochastic}---an approach that is shown to be optimal in terms of IFO complexity in the centralized setting---in a randomly activated manner to further improve computational efficiency when the local sample size is limited.
To reduce communication, \myalg employs gradient tracking \cite{zhu2010discrete} with a few mixing rounds per iteration, which helps accelerate the convergence through better information sharing \cite{li2020communication}; the extra mixing scheme can be implemented using Chebyshev acceleration \cite{arioli2014chebyshev} to further improve the communication efficiency. In a nutshell, to find an $\epsilon$-approximate first-order stationary points, i.e. $ \E \big\| \nabla f(\xout) \big\|^2_2\leq \epsilon$, where $\xout$ is the output of \myalg, and the expectation  is taken with respect to the randomness of the algorithm, \myalg requires:
\begin{itemize}
\item $O \big(m + (m/n)^{1/2} L / \epsilon \big)$ per-agent IFO calls,\footnote{The big-$O$ notation is defined in \Cref{sub:paper_organization_and_notation}.} which is {\em network-independent}; and
 
\item $O\Big(\frac{\log \big((n/m)^{1/2} + 2 \big)}{(1 - \alpha)^{1/2}} \cdot \big( (mn)^{1/2} + \frac {L}{\epsilon} \big) \Big)$ rounds of communication, 
\end{itemize}
where $L$ is the smoothness parameter of the sample loss, $\alpha \in [0,1) $ is the mixing rate of the network topology, $n$ is the number of agents, and $m=N/n$ is the local sample size.

\begin{table}[tb]
    \centering
    \resizebox{\textwidth}{!}{%
\begin{tabular}{c||c|c|c}
    \toprule

    \textcolor{rgb:red,255;green,0;blue,0}{}
    Algorithms &Setting &  Per-agent IFO Complexity
    & Communication Rounds
    \\

    \hline\hline

    SVRG
    &  \multirow{2}{*}{centralized} & \multirow{2}{*}{$N + \frac{N^{2/3} L}{\epsilon}$}
    & \multirow{2}{*}{n/a}
    \\
    \cite{allen2016variance,reddi2016stochastic}    & &  & \\
    \hline
    SCSG/SVRG+
    &  \multirow{2}{*}{centralized} & \multirow{2}{*}{$N + \frac{N^{2/3} L}{\epsilon}$}
    & \multirow{2}{*}{n/a}
    \\
     \cite{lei2019non,li2018simple}  & &  & \\
    \hline

    SNVRG
    &  \multirow{2}{*}{centralized} &  \multirow{2}{*}{$N + \frac{N^{1/2} L}{\epsilon}$} %
    &  \multirow{2}{*}{n/a} %
    \\
    \cite{zhou2020stochastic} & &  & \\
    \hline
    SARAH/SPIDER/SpiderBoost 
    &  \multirow{2}{*}{centralized} &  \multirow{2}{*}{$N + \frac{N^{1/2} L}{\epsilon}$} %
    &  \multirow{2}{*}{n/a} %
    \\
    \cite{nguyen2019finite,fang2018spider,wang2019spiderboost} & &  & \\
    \hline

    SSRGD/ZeroSARAH/PAGE 
    &  \multirow{2}{*}{centralized} &  \multirow{2}{*}{$N + \frac{N^{1/2} L}{\epsilon}$} %
    &  \multirow{2}{*}{n/a} %
    \\
    \cite{li2019ssrgd,li2021zerosarah,li2021page} & &  & \\
    \hline

    D-GET 
    & \multirow{2}{*}{decentralized} & \multirow{2}{*}{$m + \frac{1}{(1 - \alpha)^2} \cdot \frac{m^{1/2} L}{\epsilon}$}
    & \multirow{2}{*}{Same as IFO} \\
    \cite{sun2020improving} & & & \\
    \hline

    GT-SARAH
    & \multirow{2}{*}{decentralized} & \multirow{2}{*}{$m + \max \Big( \frac{1}{(1 - \alpha)^2}, \big(\frac mn \big)^{1/2}, \frac{(m/n + 1)^{1/3}}{1 - \alpha} \Big) \cdot \frac{L}{\epsilon}$}
    & \multirow{2}{*}{ Same as IFO }
    \\
     \cite{xin2020near} & & & \\
    \hline

    \hline
    \myalg 
    & \multirow{2}{*}{decentralized} & \multirow{2}{*}{ $m + \frac{(m/n)^{1/2} L}{\epsilon}$} 
    & \multirow{2}{*}{$\frac{1}{(1 - \alpha)^{1/2}} \cdot \Big( (mn)^{1/2} + \frac {L}{\epsilon} \Big)$}
    \\
(this paper) & & & \\
    \bottomrule
    \hline

\end{tabular}
}
    \caption{The per-agent IFO complexities and communication complexities to find $\epsilon$-approximate first-order stationary points by stochastic variance-reduced algorithms for nonconvex finite-sum problems. 
        The algorithms listed in the first three rows are designed for the centralized setting,
        and the remaining D-GET, GT-SARAH and our \myalg are in the decentralized setting.
        Here, $n$ is the number of agents, $m=N/n$ is the local sample size, $L$ is the smoothness parameter of the sample loss, and $\alpha \in [0,1) $ is the mixing rate of the network topology.      The big-$O$ notations and logarithmic terms are omitted for simplicity.
    \label{table:1st}
    }
\end{table}

\paragraph{Comparisons with existing algorithms.}
 \Cref{table:1st} summarizes the convergence guarantees of representative stochastic variance-reduced algorithms for finding first-order stationary points across centralized and decentralized communication settings. 
\begin{itemize}
\item In terms of the computation complexity, the overall IFO complexity of \myalg---when summed over all agents---becomes
$$ n \cdot O \big(m + (m/n)^{1/2} L / \epsilon \big) =  O \big(mn + (mn)^{1/2} L/ \epsilon \big)= O \big( N + N^{1/2} L/ \epsilon \big) ,$$
matching the optimal IFO complexity of centralized algorithms (e.g., SPIDER~\cite{fang2018spider}, PAGE~\cite{li2021page}) and distributed server/client algorithms (e.g., D-ZeroSARAH~\cite{li2021zerosarah}).
However, the state-of-the-art
decentralized algorithm GT-SARAH \cite{xin2020near} still did not achieve this optimal IFO complexity for most situations (see \Cref{table:1st}).
To the best of our knowledge, \myalg is the first algorithm to achieve the optimal IFO complexity for the decentralized setting regardless of network topology and sample size.

\item When it comes to the communication complexity, it is observed that the communication rounds of \myalg can be decomposed into the sum of an $\epsilon$-independent term and an $\epsilon$-dependent term (up to a logarithmic factor), i.e.,
    $$  \underbrace{ \frac{1}{(1 - \alpha)^{1/2}} \cdot (mn)^{1/2}  }_{\epsilon-\mathsf{independent}}  + \underbrace{ \frac{1}{(1 - \alpha)^{1/2}}   \cdot \frac{L}{\epsilon}}_{\epsilon-\mathsf{dependent}} ;  $$
similar decompositions also apply to competing decentralized algorithms. 
\myalg significantly improves the $\epsilon$-dependent term of D-GET and GT-SARAH by at least a factor of $\frac{1}{(1-\alpha)^{3/2}}$, and therefore, saves more communications over poorly-connected networks. Further, the $\epsilon$-independent term of \myalg is also smaller than that of D-GET/GT-SARAH as long as the local sample size is sufficiently large, i.e. $m = \Omega \big( \frac{n}{1-\alpha} \big) $,
which also holds for a wide variety of application scenarios. To gain further insights in terms of the communication savings of \myalg,
\Cref{table:1st_communication} further compares the communication complexities of decentralized algorithms for finding first-order stationary points under three common network settings.

\end{itemize}

\begin{table}[tb]
    \centering
    \resizebox{\textwidth}{!}{%

\begin{tabular}{c||c|c|c}
	
	\toprule

	& Erd\H{o}s-R\'enyi graph
	& 2-D grid graph
	& Path graph
	\\

	\hline
	\hline

	$1 - \alpha$
	& \multirow{2}{*}{$1$}
	& \multirow{2}{*}{$\frac{1}{n \log n}$}
	& \multirow{2}{*}{$\frac{1}{n^2}$}
	\\
	(spectral gap)
	& 
	& 
	& 
	\\
	
	\hline
	
	D-GET
	& \multirow{2}{*}{$m + \frac{m^{1/2} L}{\epsilon}$}
    & \multirow{2}{*}{$m + \frac{m^{1/2} n^2 L}{\epsilon}$}
	& \multirow{2}{*}{$m + \frac{m^{1/2} n^4L}{\epsilon}$}
	\\
	
	\cite{sun2020improving}
	& 
	& 
	& 
	\\
	
	\hline
	
	GT-SARAH
	& \multirow{2}{*}{$m + \max \Big\{1,~ \big(\frac{m}{n} \big)^{1/3},~ \big(\frac mn \big)^{1/2} \Big\} \cdot \frac{L}{\epsilon}$}
	& \multirow{2}{*}{$m + \max \Big\{n^2,~ m^{1/3}n^{2/3},~ \big(\frac mn \big)^{1/2} \Big\} \cdot \frac{L}{\epsilon}$}
	& \multirow{2}{*}{$m + \max \Big\{n^4,~ m^{1/3}n^{5/3},~ \big(\frac mn \big)^{1/2} \Big\} \cdot \frac{L}{\epsilon}$}
	\\
	
	\cite{xin2020near}
	& 
	& 
	& 
	\\
	
	\hline
	
	\myalg
	& \multirow{2}{*}{$(mn)^{1/2} + \frac {L}{\epsilon}$}
    & \multirow{2}{*}{$m^{1/2}n  + \frac {n^{1/2} L}{\epsilon} $}
	& \multirow{2}{*}{$(m n^3)^{1/2}  + \frac {nL}{\epsilon} $}
	\\
	
	(this paper)
	& 
	& 
	& 
	\\
	
	\hline 
	\hline 
	
 Improvement factors   
    & \multirow{2}{*}{$\big(\frac mn \big)^{1/2}$   }
    & \multirow{2}{*}{$\frac{m^{1/2}}{n}$   }
    & \multirow{2}{*}{$\frac{m^{1/2}}{n^{3/2}}$    }
	\\
	for $\epsilon$-independent term
	& 
	& 
	& 
	\\

	\hline
	
	Improvement factors
	& \multirow{2}{*}{$\max \Big\{1,~ \big(\frac{m}{n} \big)^{1/3},~\big(\frac mn \big)^{1/2} \Big\}$ }
    & \multirow{2}{*}{$\max \Big\{ n^{3/2},~ m^{1/3}n^{1/6},~ \frac{m^{1/2}}{n} \Big\}$}
	& \multirow{2}{*}{$\max \Big\{ n^3,~ m^{1/3}n^{2/3},~ \frac{m^{1/2}}{n^{3/2}}\Big\}$}
	\\
	for $\epsilon$-dependent term
	& 
	& 
	& 
	\\
	
	\bottomrule
	
	\hline
	
\end{tabular}
}
    \caption{Detailed comparisons of the communication complexities of D-GET, GT-SARAH and \myalg under three graph topologies, where the last two rows delineate the improve factors of \myalg over existing algorithms. The communication savings become significant especially when $m = \Omega \big( \frac{n}{1-\alpha} \big) $.
        The complexities are simplified by plugging the bound on the spectral gap $1-\alpha$ from \cite[Proposition 5]{nedic2018network}. 
        Here, $n$ is the number of agents, $m=N/n$ is the local sample size, $L$ is the smoothness parameter of the sample loss, and $\alpha \in [0,1) $ is the mixing rate of the network topology.
        The big-$O$ notations and logarithmic terms are omitted for simplicity.
    \label{table:1st_communication}
    }
\end{table}
In sum, \myalg harnesses the ideas of variance reduction, gradient tracking and extra mixing
in a sophisticated manner to achieve a scalable decentralized algorithm for nonconvex empirical risk minimization that is competitive in both computation and communication over existing approaches.

\subsection{Additional related works}%
\label{sub:related_works}

Decentralized optimization and learning have been studied extensively, with contemporary emphasis on the capabilities to scale gracefully to large-scale problems --- both in terms of the size of the data and the size of the network.
For the conciseness of the paper, we focus our discussions on the most relevant literature and refer interested readers to recent overviews \cite{nokleby2020scaling,xin2020general,xin2020decentralized} for further references.

\paragraph{Stochastic variance-reduced methods.} Many variants of stochastic variance-reduced gradient based methods have been proposed for finite-sum optimization for finding first-order stationary points, including but not limited to
SVRG \cite{Johnson2013,allen2016variance,reddi2016stochastic}, SCSG \cite{lei2019non}, SVRG+ \cite{li2018simple}, SAGA \cite{defazio2014saga},
SARAH \cite{nguyen2017sarah,nguyen2019finite}, SPIDER \cite{fang2018spider}, SpiderBoost \cite{wang2019spiderboost}, SSRGD \cite{li2019ssrgd}, ZeroSARAH \cite{li2021zerosarah} and PAGE \cite{li2021page,li2021short}.
SVRG/SVRG+/SCSG/SAGA utilize stochastic variance-reduced gradients as a corrected estimator of the full gradient, but can only achieve a sub-optimal IFO complexity of $O(N + N^{2/3} L / \epsilon)$. Other algorithms such as SARAH, SPIDER, SpiderBoost, SSRGD and PAGE adopt stochastic recursive gradients to improve the IFO complexity to $O(N + N^{1/2} L / \epsilon)$, which is optimal indicated by the lower bound provided in \cite{fang2018spider,li2021page}.
\myalg also utilizes the stochastic recursive gradients to perform variance reduction,
which results in the optimal IFO complexity for finding first-order stationary points.

\paragraph{Decentralized stochastic nonconvex optimization.}  There has been a flurry of recent activities in decentralized nonconvex optimization in both the server/client  setting and the network setting. In the server/client setting, \cite{cen2019convergence} simplifies the approaches in \cite{lee2017distributed} for distributing stochastic variance-reduced algorithms without requiring sampling extra data. In particular, D-SARAH \cite{cen2019convergence} extends SARAH to the server/client setting but with a slightly worse IFO complexity and a sample-independent communication complexity. 
D-ZeroSARAH \cite{li2021zerosarah} obtains the optimal IFO complexity in the server/client setting.
In the network setting, D-PSGD \cite{lian2017can} and SGP \cite{assran2019stochastic} extend stochastic gradient descent (SGD) to solve the nonconvex decentralized expectation minimization problems with sub-optimal rates.
However, due to the noisy stochastic gradients,
D-PSGD can only use diminishing step size to ensure convergence, and
SGP uses a small step size on the order of $1/K$, where $K$ denotes the total iterations.
$D^2$ \cite{tang2018d2} introduces a variance-reduced correction term to D-PSGD,
which allows a constant step size and hence reaches a better convergence rate.

Gradient tracking \cite{zhu2010discrete,qu2018harnessing} provides a systematic approach to estimate the global gradient at each agent,
which allows one to easily design decentralized optimization algorithms based on existing centralized algorithms. This idea is applied in \cite{zhang2020decentralized} to extend SGD to the decentralized setting, and in \cite{li2020communication} to extend quasi-Newton algorithms as well as stochastic variance-reduced algorithms, with performance guarantees for optimizing strongly convex functions.
GT-SAGA \cite{xin2020fast} further uses SAGA-style updates and reaches a convergence rate that matches SAGA \cite{defazio2014saga,reddi2016fasta}.
However, GT-SAGA requires to store a variable table, which leads to a high memory complexity.
D-GET \cite{sun2020improving} and GT-SARAH \cite{xin2020near} adopt equivalent recursive local gradient estimators to enable the use of constant step sizes without extra memory usage. The IFO complexity of GT-SARAH is optimal in the restrictive range $m\gtrsim \frac{n}{(1-\alpha)^6}$, while \myalg achieves the optimal IFO over all parameter regimes.  

In addition to variance reduction techniques,
performing multiple mixing steps between local updates can greatly improve the dependence of the network in convergence rates,
which is equivalent of communicating over a better-connected communication graph for the agents, which in turn leads to a faster convergence (and a better overall efficiency) due to better information mixing.
This technique is applied by a number of recent literature including \cite{berahas2018balancing,pan2019d,berahas2020convergence,li2020communication,hashemi2020benefits,iakovidou2021s}, and its effectiveness is verified both in theory and experiments.
Our algorithm also adopts the extra mixing steps,
which leads to better IFO complexity and communication complexity.

\subsection{Paper organization and notation}
\label{sub:paper_organization_and_notation}

\Cref{sec:preliminaries} introduces preliminary concepts and the algorithm development,
\Cref{sec:algorithm} shows the theoretical performance guarantees for \myalg,
\Cref{sec:numerical} provides numerical evidence to support the analysis,
and \Cref{sec:conclusion} concludes the paper.
Proofs and experiment settings are postponed to appendices.

Throughout this paper, we use boldface letters to represent matrices and vectors.
We use $\| \cdot \|_{\mathsf{op}}$ for matrix operator norm,
$\otimes$ for the Kronecker product,
$\bI_n$ for the $n$-dimensional identity matrix
and $\bm1_n$ for the $n$-dimensional all-one vector. 
For two real functions $f(\cdot)$ and $g(\cdot)$ defined on $\R^+$,
we say $f(x) = O \big( g(x) \big)$ or $f(x) \lesssim g(x)$ if there exists some universal constant $M > 0$ such that
$f(x)  \leq M g(x)$. The notation $f(x) =\Omega \big( g(x) \big)$ or $f(x)\gtrsim g(x)$ means $g(x) = O\big(f(x) \big)$.

\section{Preliminaries and Proposed Algorithm}%
\label{sec:preliminaries}
 
We start by describing a few useful preliminary concepts and definitions in \Cref{sec:preliminary},
then present the proposed algorithm in \Cref{sub:algorithm_development}. 

\subsection{Preliminaries}\label{sec:preliminary}

\paragraph{Mixing.}%
The information mixing between agents is conducted by updating the local information via a weighted sum of information from neighbors,
which is characterized by a mixing (gossiping) matrix. Concerning this matrix is an important quantity called the mixing rate, defined in \Cref{definition:mixing_matrix}.
\begin{definition}[Mixing matrix and mixing rate]
    \label{definition:mixing_matrix}
    The {\em mixing matrix} is a matrix $\bW = [w_{ij}] \in \R^{n \times n}$,
    such that $w_{ij} = 0$ if agent $i$ and $j$ are not connected according to the communication graph $\cG$. Furthermore,
    $\bW \bm1_n = \bm1_n$ and $\bW^\top \bm1_n = \bm1_n$.
The {\em mixing rate} of a mixing matrix $\bW$ is defined as
    \begin{align}
        \alpha := \big\|\bW  - \tfrac1n \bm1_n\bm1_n^\top \big\|_{\mathsf{op}}.
        \label{eq:def_alpha0}
    \end{align}
\end{definition}

The mixing rate indicates the speed of information shared across the network.
For example,
for a fully-connected network,
choosing $\bW = \tfrac{1}{n}\bm1_n\bm1_n^\top$ leads to $\alpha = 0$.
For general networks and mixing matrices,
\cite[Proposition 5]{nedic2018network} provides comprehensive bounds on $1-\alpha$---also known as the spectral gap---for various graphs.
In practice,
FDLA matrices \cite{Xiao2004} are more favorable because it can achieve a much smaller mixing rate,
but they usually contain negative elements and are not symmetric.
Different from other algorithms that require the mixing matrix to be doubly-stochastic,
our analysis can handle arbitrary mixing matrices as long as their row/column sums equal to one.

\paragraph{Dynamic average consensus.} It has been well understood by now that using a naive mixing of local information merely, e.g. the local gradients of neighboring agents, does not lead to fast convergence of decentralized extensions of centralized methods \cite{nedic2009distributed,shi2015extra}. This is due to the fact that the quantity of interest in solving decentralized optimization problems is often iteration-varying, which naive mixing is unable to track; consequently, an accumulation of errors leads to either slow convergence or poor accuracy. Fortunately, the general scheme of 
dynamic average consensus \cite{zhu2010discrete} proves to be extremely effective in this regard to track the dynamic average of local variables over the course of iterative algorithms, and has been applied to extend many central algorithms to decentralized settings, e.g.  \cite{nedic2017achieving,qu2018harnessing,di2016next,li2020communication}. This idea, also known as ``gradient tracking'' in the literature, essentially adds a correction term to the naive information mixing, which we will employ in the communication stage of the proposed algorithm to track the dynamic average of local gradients.

\paragraph{Stochastic recursive gradient methods.} Stochastic recursive gradients methods \cite{nguyen2019finite,fang2018spider,wang2019spiderboost,li2019ssrgd} achieve the optimal IFO complexity in the centralized setting for nonconvex finite-sum optimization, which make it natural to adapt them to the decentralized setting with the hope of maintaining the appealing IFO complexity. Roughly speaking, these methods use a nested loop structure to iteratively refine the parameter, where 1) a global gradient evaluation is performed at each outer loop, and 2) a stochastic recursive gradient estimator is used to calculate the gradient and update the parameter at each inner loop. In the proposed \myalg algorithm, this nested loop structure lends itself to a natural decentralized
scheme, as will be seen momentarily.

\paragraph{Additional notation.} For convenience of presentation, define the stacked vector $\x \in \R^{nd}$ and its average over all agents $\bbx \in \R^d$ as
\begin{align}
	\label{eq:dev_vectors}
        \x := \big[ \x_1^\top, \cdots, \x_n^\top \big]^{\top},
        \quad
        \bbx = \frac1n \sum_{i=1}^n \x_i.
\end{align}
The vectors $\bs$, $\bbs$, $\bu$, $\bbu$, $\bv$ and $\bbv$ are defined in the same fashion.
In addition, for a stacked vector $\x \in \R^{nd}$,
we introduce the distributed gradient $\nabla F(\x) \in \R^{nd}$ as
\begin{align}
	\label{eq:def_gradient}
    \nabla F(\x) := [\nabla f_1(\x_1)^\top, \cdots, \nabla f_n(\x_n)^\top]^\top .
\end{align}

\subsection{The \myalg Algorithm}
\label{sub:algorithm_development}

Detailed in \Cref{alg:network_sarah}, we propose a novel decentralized stochastic optimization algorithm, dubbed \myalg, for finding first-order order stationary points of nonconvex finite-sum problems. Motivated by stochastic recursive gradient methods in the centralized setting,
\myalg has a nested loop structure:
\begin{enumerate} 
\item The outer loop adopts dynamic average consensus to estimate and track the global gradient $\nabla F(\bx^{(t)})$ at each agent in \eqref{eq:gradient_tracking}, where $\bx^{(t)}$ is the stacked parameter estimate (cf.~\eqref{eq:def_gradient}). This helps to ``reset'' the stochastic gradient to a less noisy starting gradient $\bv^{(t), 0}=\bs^{(t)}$ of the inner loop. A key property of \eqref{eq:gradient_tracking}---which is a direct consequence of dynamic average consensus---is that the average of $\bs^{(t)}$ equals to the dynamic average of local gradients, i.e. $\bbs^{(t)} = \tfrac1n \sum_{i\in[n]} \bs_i^{(t)} = \tfrac1n \sum_{i\in[n]} \nabla f_i(\x_i^{(t)})$.

\item The inner loop refines the parameter estimate $\bu^{(t), 0} = \x^{(t)}$ by performing randomly activated stochastic recursive gradient updates \eqref{eq:inner_loop}, where the stochastic recursive gradient $\bg^{(t), s}$ is updated in \eqref{eq:inner_loop_sg} via sampling mini-batches from activated agents' local datasets.

\end{enumerate}
To complete the last mile, inspired by \cite{li2020communication}, we allow \myalg to perform a  few rounds of mixing or gossiping whenever communication takes place, to enable better information sharing and faster convergence. Specifically, 
\myalg performs $\KO$ and $\KI$ mixing steps for the outer and inner loops respectively per iteration,
which is equivalent to using 
$$ \bW_{\mathsf{out}} = \bW^{\KO} \qquad \mbox{and} \qquad \bW_{\mathsf{in}} = \bW^{\KI}$$ as mixing matrices, and correspondingly a network with better connectivity; see \eqref{eq:gradient_tracking}, \eqref{eq:inner_loop_var} and \eqref{eq:inner_loop_grad}. 
Note that \Cref{alg:network_sarah} is written in matrix notation, where the mixing steps are described by $\bW_{\mathsf{in}} \otimes \bI_n$ or $\bW_{\mathsf{out}} \otimes \bI_{n}$ and applied to all agents simultaneously. The extra mixing steps can be implemented by Chebyshev acceleration \cite{arioli2014chebyshev} with improved communication efficiency.

\begin{algorithm}[h]
 \caption{\myalg for decentralized nonconvex finite-sum optimization}
\label{alg:network_sarah}
\begin{algorithmic}[1]
    \STATE {\textbf{input:} initial parameter $\bbx^{(0)}$,
        step size $\eta$,
        activation probability $p$,
        batch size $b$,
        number of outer loops $T$,
        number of inner loops $S$,
        and number of communication (extra mixing) steps $\KI$ and $\KO$.
    }

    \STATE {\textbf{initialization:} set
    $\x_i^{(0)} = \bbx^{(0)}$ and
    $\bs_i^{(0)} = \nabla f(\bbx^{(0)})$ for all agents $1 \leq i \leq n$. }

    \FOR {$t = 1, \ldots, T$}
        \STATE{Set the new parameter estimate $\x^{(t)} = \bu^{(t-1), S}$.}
        
        \STATE{Update the global gradient estimate by aggregated local information and gradient tracking:} 
        \begin{align}
            \bs^{(t)} =& \WKO \Big( \bs^{(t-1)} + \nabla F \big( \x^{(t)} \big) - \nabla F \big( \x^{(t-1)} \big)  \Big)
            \label{eq:gradient_tracking}
        \end{align}

        \STATE{Set $\bu^{(t), 0} = \bx^{(t)}$ and $\bv^{(t), 0} = \bs^{(t)}$.}

            \FOR{$s=1,...,S$}

            \STATE Each agent $i$ samples a mini-batch $\cZ_i^{(t), s}$ of size $b$ from $\cM_i$ uniformly at random, $\lambda_i^{(t), s} \sim \mathcal B(p)$ where $\mathcal B(p)$ denotes the Bernoulli distribution with parameter $p$,\footnotemark~and then performs the following updates:
                \begin{subequations}
                \label{eq:inner_loop}
                \begin{align}
                    \bu^{(t), s} &= \WKI (\bu^{(t), s-1} - \eta \bv^{(t), s-1}),
                    \label{eq:inner_loop_var} \\
                    \bg_i^{(t),s} & = \frac{ \lambda_i^{(t), s}}{pb} \sum_{\bz_i \in \cZ_i^{(t), s}} \Big( \nabla \ell (\bu_i^{(t), s}; \bz_i) 
                                  - \nabla \ell (\bu_i^{(t), s-1}; \bz_i) \Big) + \bv_i^{(t), s-1} ,
                                  \label{eq:inner_loop_sg} \\
                    \bv^{(t), s} &= \WKI \bg^{(t),s} .
                    \label{eq:inner_loop_grad}
                \end{align}
                \end{subequations}
                \vspace{-15pt}

            \ENDFOR

    \ENDFOR
    \STATE {\textbf{output:} ~$\xout \sim \text{Uniform} (\{\bu_i^{(t), s-1} | i \in [n], t \in [T], s \in [S] \})$.}	
\end{algorithmic}
\end{algorithm}
\footnotetext{The stochastic gradients will not be computed if $\lambda_i^{(t), s} = 0$.}

Compared with existing decentralized algorithms based on stochastic variance-reduced algorithms such as D-GET \cite{sun2020improving} and GT-SARAH \cite{xin2020near},  
\myalg utilizes different gradient estimators and communication protocols:
First, \myalg produces a sequence of reference points $x^{(t)}$---which converge to a global first-order stationary point---to ``restart'' the inner loops periodically using fresher information; secondly, the communication and computation in \myalg are paced differently due to the introduction of extra mixing, which allow a more flexible trade-off schemes between different types of resources; last but not least, the random activation of stochastic recursive gradient updates further saves local computation, especially when the local sample size is small compared to the number of agents.
%

%
%

%

%
\section{Performance Guarantees}
\label{sec:algorithm}

This section presents the performance guarantees of \myalg for finding first-order stationary points of the global objective function $f(\cdot)$.

\subsection{Assumptions}
We first introduce \Cref{assumption:lipschitz_gradient} and \Cref{assumption:optimality_lower_bounded},
which are standard assumptions imposed on the loss function.
\Cref{assumption:lipschitz_gradient} implies that all local objective functions $f_i(\cdot)$ and the global objective function $f(\cdot)$ also have Lipschitz gradients,
and \Cref{assumption:optimality_lower_bounded} guarantees the absence of trivial solutions.
\begin{assumption}[Lipschitz gradient]
    \label{assumption:lipschitz_gradient}
    The sample loss function $\ell(\x; \bz)$ has $L$-Lipschitz gradients for all $\bz \in \cM$ and $\x \in \R^d$,
    namely,
    $\big\| \nabla \ell (\x; \bz) - \nabla \ell (\x'; \bz) \big\|_2 \leq L \| \x - \x' \|_2$, 
    $\forall \x, \x' \in \R^d$ and $\bz \in \cM$.
\end{assumption}

\begin{assumption}[Function boundedness]
    \label{assumption:optimality_lower_bounded}
    The global objective function $f(\cdot)$ is bounded below,
    i.e., $f^* = \inf_{\bx\in\mathbb{R}^d} f(\bx ) > -\infty$.
\end{assumption}

Due to the nonconvexity, first-order algorithms are generally guaranteed to converge to only first-order stationary points of the global loss function $f(\cdot)$, defined below in \Cref{definition:1st}.
\begin{definition}[First-order stationary point]
    \label{definition:1st}
    A point $\x \in \R^d$ is called an $\epsilon$-approximate first-order stationary point of a differentiable function $f(\cdot)$ if
    \begin{align*}
        \| \nabla f(\x) \|_2^2 \leq \epsilon .
    \end{align*} 
\end{definition}

\subsection{Main theorem}
\Cref{theorem:network_sarah_non_convex}, whose proof is deferred to Appendix~\ref{sec:proof_of_theorem_ref}, shows that \myalg converges in expectation to an approximate first-order stationary point, under suitable parameter choices.

\begin{theorem}[First-order optimality]
    \label{theorem:network_sarah_non_convex}
    Assume \Cref{assumption:lipschitz_gradient,assumption:optimality_lower_bounded} hold.
    Set $p \in (0, 1]$,
    $\KI$, $\KO$, $S$, $b$ and $\eta$ to be positive and satisfy
    \begin{equation}
    \label{eq:step_size_condition}
    \alpha^{\KI} \leq p
    \text{~~~~and~~~~}
    \eta L \leq \frac{(1 - \alpha^{\KI})^3 (1 - \alpha^{\KO})}{10 \big(1 + \alpha^{\KI} \alpha^{\KO} \sqrt{npb} \big)   \big( \sqrt{S/(npb)} + 1 \big)}.
    \end{equation}
The output produced by \Cref{alg:network_sarah} satisfies
    \begin{align}
        & \E \big\| \nabla f(\xout) \big\|^2_2
        < \frac{4}{\eta TS} \Big( \E[ f(\bbx^{(0)})] - f^* \Big) .
        \label{eq:theorem_1}
    \end{align}
\end{theorem}

If there is only one agent, i.e. $n= 1$,
the mixing rate will be $\alpha=0$,
we can choose $\KI = \KO = p =1$,
and \Cref{theorem:network_sarah_non_convex} reduces to \cite[Theorem 1]{nguyen2019finite}, its counterpart in the centralized setting.
For general decentralized settings with arbitrary mixing schedules,
\Cref{theorem:network_sarah_non_convex} provides a comprehensive characterization of the convergence rate, where an $\epsilon$-approximate first-order stationary point can be found in expectation in a total of 
$$ TS = O\left(\frac{\E[ f(\bbx^{(0)})] - f^*}{\eta \epsilon} \right)$$ 
iterations; here, $T$ is the number of outer iterations and $S$ is the number of inner iterations. Clearly, a larger step size $\eta$, as allowable by \eqref{eq:step_size_condition}, hints on a smaller iteration complexity, and hence a smaller IFO complexity.

There are two conditions in \eqref{eq:step_size_condition}. On one end, $\KI$ needs to be large enough (i.e., perform more rounds of extra mixing) to counter the effect when $p$ is small (i.e., we compute less stochastic gradients every iteration), or when $\alpha$ is close to $1$ (i.e., the network is poorly connected).
On the other end, the step size $\eta$ needs to be small enough to account for the requirement of the step size in the centralized setting, as well as the effect of imperfect communication due to decentralization.
For well-connected networks where $\alpha \ll 1$, the terms introduced by the decentralized setting will diminish---indicating the iteration complexity is close to that of the centralized setting.
For poorly-connected networks, carefully designing the mixing matrix and other parameters can ensure a desirable trade-off between convergence speed and communication cost. 
The following corollary provides specific parameter choices for \myalg to achieve the optimal per-agent IFO complexity.
The proof is deferred to \Cref{sec:proof_of_corollay_1}.

\begin{corollary}[Complexity for finding first-order stationary points]
    \label{corollary:network_sarah_non_convex_two}
    Under conditions of \Cref{theorem:network_sarah_non_convex},
    set
    $S = \Big\lceil \sqrt{mn} \Big\rceil$,
    $b = \left\lceil \sqrt{m/n} \right\rceil$,
    $p = \frac{\sqrt{m / n}}{\left\lceil \sqrt{m/n} \right\rceil}$,
    $\KO = \left\lceil \frac{\log (\sqrt{npb} + 1)}{(1 - \alpha)^{1/2}} \right\rceil$,
    $\KI = \left\lceil \frac{\log (2 / p)}{(1 - \alpha)^{1/2}} \right\rceil$, $\eta = \frac{1}{640 L}$,
 and  implement the mixing steps using Chebyshev's acceleration \cite{arioli2014chebyshev}.
    To reach an $\epsilon$-approximate first-order stationary point,
    in expectation, \myalg takes $O\Big( m + \frac{(m/n)^{1/2} L}{\epsilon} \Big)$ IFO calls per agent, and
 $O\Big(\frac{\log \big((n/m)^{1/2} + 2 \big)}{(1 - \alpha)^{1/2}} \cdot \big( (mn)^{1/2} + \frac {L}{\epsilon} \big) \Big)$ rounds of communication.
\end{corollary}
As elaborated in Section~\ref{sub:contributions}, \myalg achieves a network-independent IFO complexity that matches the optimal complexity in the centralized setting. In addition, when the accuracy $\epsilon\lesssim   L/(mn)^{1/2}  $,
\myalg reaches a communication complexity of $O\big( \frac{1}{(1 - \alpha)^{1/2}} \cdot \frac{L}{\epsilon} \big)$, which is independent of the sample size.

It is worthwhile to further highlight the role of the random activation probability $p$ in achieving the optimal IFO by allowing ``fractional'' batch size. Note that the batch size is set as $b =\left\lceil \sqrt{m/n} \right\rceil$, where $m$ is the local sample size, and $n$ is the number of agents.
\begin{enumerate}
\item When the local sample size is large, i.e. $m\geq n$, we can approximate $b \approx \sqrt{m / n}$ and $p\approx 1$. In fact, Corollary~\ref{corollary:network_sarah_non_convex_two} continues to hold with $p =1$ in this regime. 
\item However, when the number of agents is large, i.e. $n>m$, the batch size $b=1$ and $p=\sqrt{m / n} <1$, which mitigates the potential computation waste by only selecting a subset of agents to perform local computation, compared to the case when we naively set $p=1$. 
\end{enumerate}
Therefore, by introducing random activation, we can view $pb = \sqrt{m/n}$ as the effective batch size at each agent, which allows fractional values and leads to the optimal IFO complexity in all scenarios.

\section{Numerical Experiments}
\label{sec:numerical}
This section provides numerical experiments 
on real datasets to evaluate our proposed algorithm \myalg with comparisons against two existing baselines: DSGD \cite{nedic2009distributed,lian2017can} and GT-SARAH \cite{xin2020near}. To allow for reproducibility, all codes can be found at 
\begin{center}
\href{https://github.com/liboyue/Network-Distributed-Algorithm}{https://github.com/liboyue/Network-Distributed-Algorithm}. 
\end{center}

For all experiments,
we set the number of agents $n = 20$,
and split the dataset uniformly at random to each agent.
In addition, since $m \gg n$ in all experiments, we set $p = 1$ for simplicity.
We run each experiment on three communication graphs with the same data assignment and starting point: Erd\"{o}s-R\`{e}nyi graph (the connectivity probability is set to $0.3$), grid graph, and path graph. The mixing matrices are chosen as the symmetric fastest distributed linear averaging (FDLA) matrices \cite{Xiao2004} generated according to different graph topologies, and the extra mixing steps are implemented by Chebyshev's acceleration \cite{arioli2014chebyshev} to save communications as described earlier.  To ensure convergence, DSGD adopts a diminishing step size schedule. All the parameters are tuned manually for best performance. We defer a detailed account of the baseline algorithms as well as parameter choices in Appendix~\ref{sec:baseline}.

\subsection{Regularized logistic regression}
\label{sub:logistic_regression}

To begin with,
we employ logistic regression with nonconvex regularization to solve a binary classification problem using the Gisette dataset.\footnote{The dataset can be accessed at \href{https://archive.ics.uci.edu/ml/datasets/Gisette}{https://archive.ics.uci.edu/ml/datasets/Gisette}.} 
We split the Gisette dataset to $n=20$ agents,
where each agent receives $m=300$ training samples of dimension $d=5000$.
The sample loss function is given as
\begin{align*}
    \ell(\x; \{ \bm f, l \})
    = -l \log \Big(\frac{1}{1 + \exp( \bx^\top \bm f)} \Big)
    + (1 - l) \log \Big( \frac{\exp(\bx^\top \bm f)}{1 + \exp(\bx^\top \bm f)} \Big)
    + \lambda \sum_{i=1}^d \frac{x_i^2}{1 + x_i^2},
\end{align*}
where $\{ \bm f, l \}$ represents a training tuple,
$\bm f \in \R^d$ is the feature vector and $l \in \{0, 1\}$ is the label, and $\lambda$ is the regularization parameter. For this experiment, we set $\lambda = 0.01$.

\begin{figure}[ht]
    \centering
    \includegraphics[width=\textwidth]{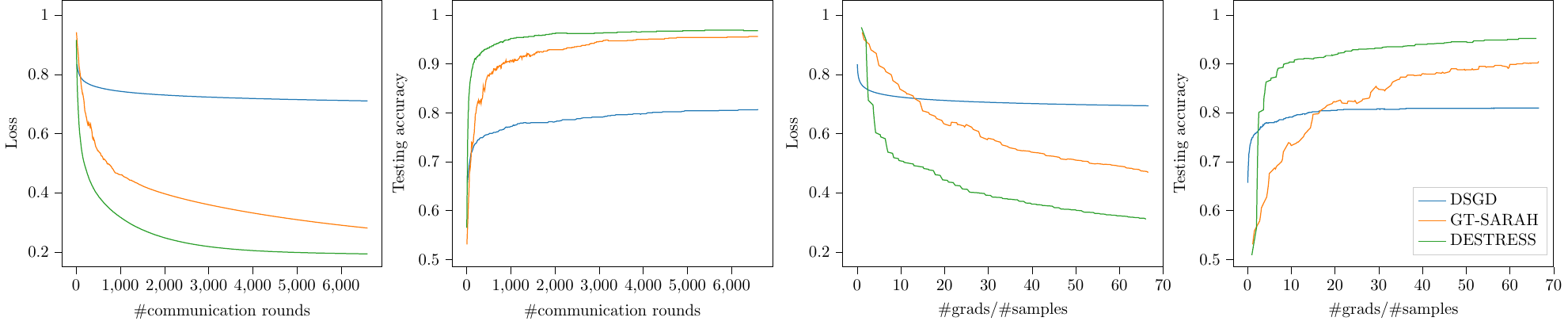} \\
    (a)  Erd\"{o}s-R\`{e}nyi  graph \\ \vspace{0.05in}
    \includegraphics[width=\textwidth]{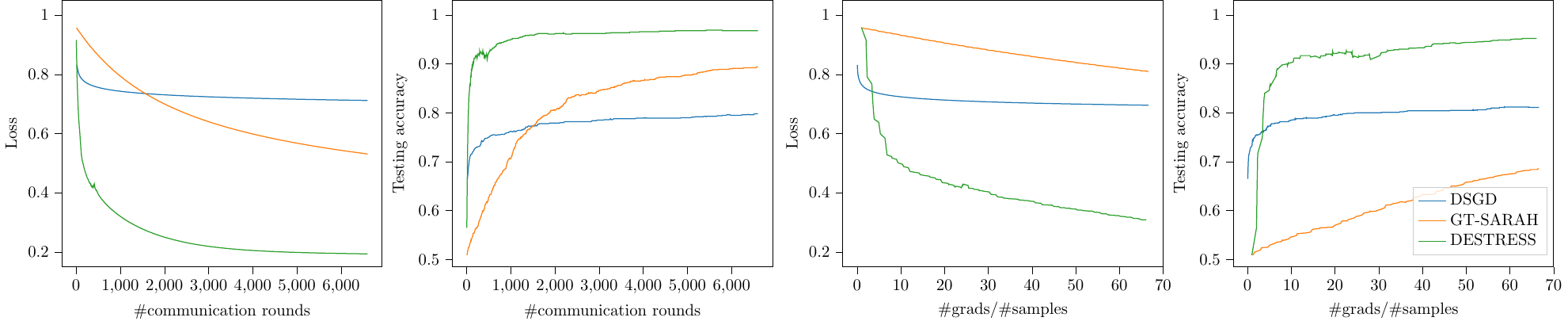}
        (b)  Grid graph \\ \vspace{0.05in}
    \includegraphics[width=\textwidth]{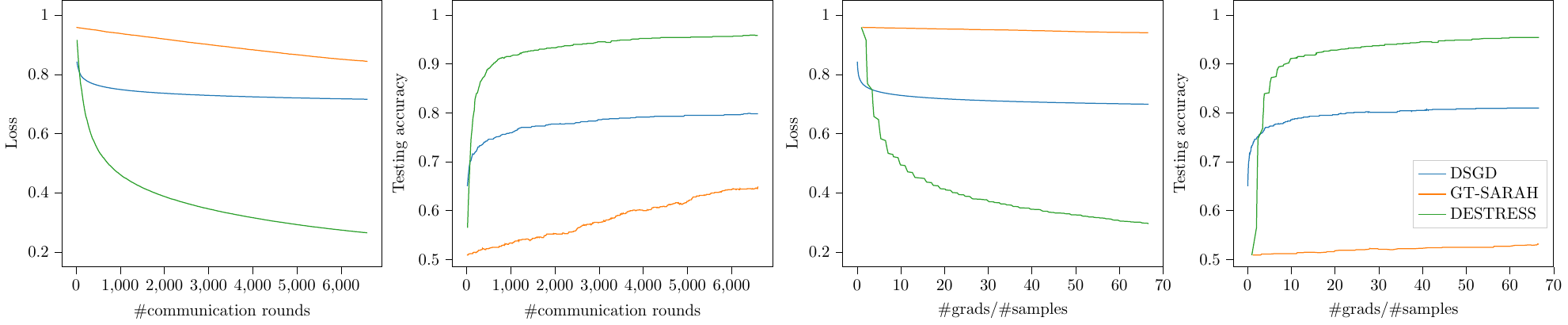}
        (c)  Path graph 
    \caption{The training loss and testing accuracy with respect to the number of communication rounds (left two panels) and gradient evaluations (right two panels) for DSGD, GT-SARAH and \myalg when training a regularized logistic regression model on the Gisette dataset.
        Due to the initial full-gradient computation, the gradient evaluations of \myalg and GT-SARAH do not start from $0$.
        \label{fig:gisette_expander}
    }
\end{figure}

\Cref{fig:gisette_expander} shows the loss and testing accuracy for all algorithms.
\myalg significantly outperforms other algorithms both in terms of communication and computation.
It is worth noting that,
DSGD converges very fast at the beginning of training, but cannot sustain the progress due to the diminishing schedule of step sizes.
On the contrary,
the variance-reduced algorithms can converge with a constant step size, and hence converge better overall.
Moreover, due to the refined gradient estimation and information mixing designs,
\myalg can bear a larger step size than GT-SARAH,
which leads to the fastest convergence and best overall performance. In addition, a larger number of extra mixing steps leads to a better performance when the graph topology becomes less connected.

\begin{figure}[ht]
    \centering
    \includegraphics[width=\textwidth]{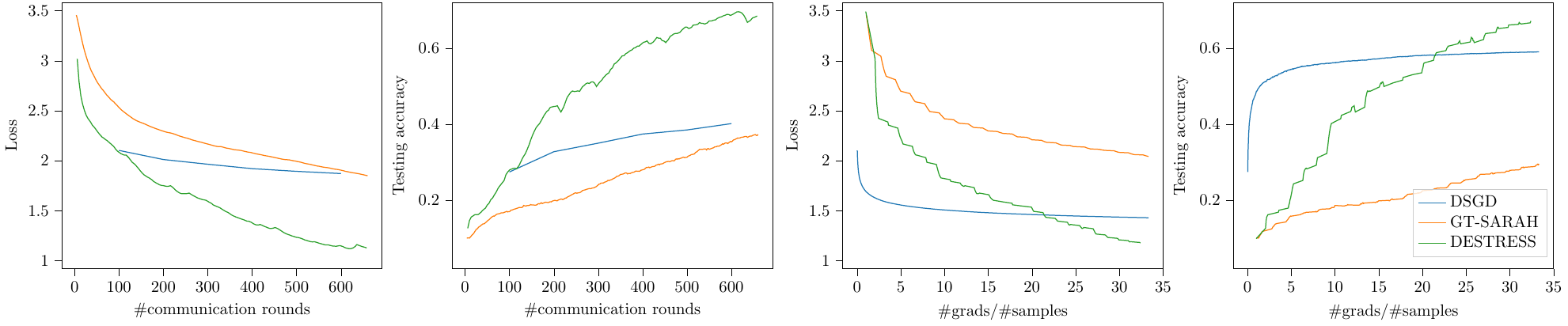} \\
    (a)  Erd\"{o}s-R\`{e}nyi  graph \\ \vspace{0.05in}
    \includegraphics[width=\textwidth]{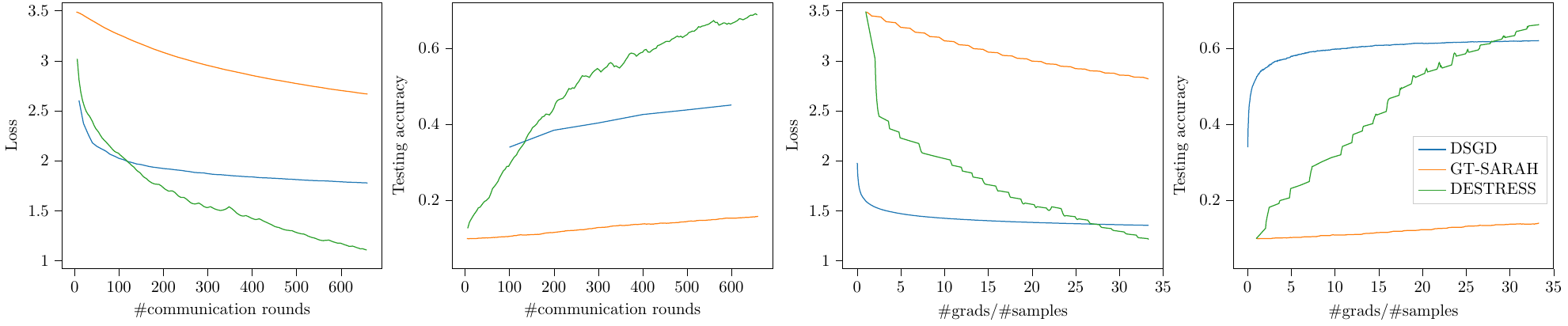}
        (b)  Grid graph \\ \vspace{0.05in}
    \includegraphics[width=\textwidth]{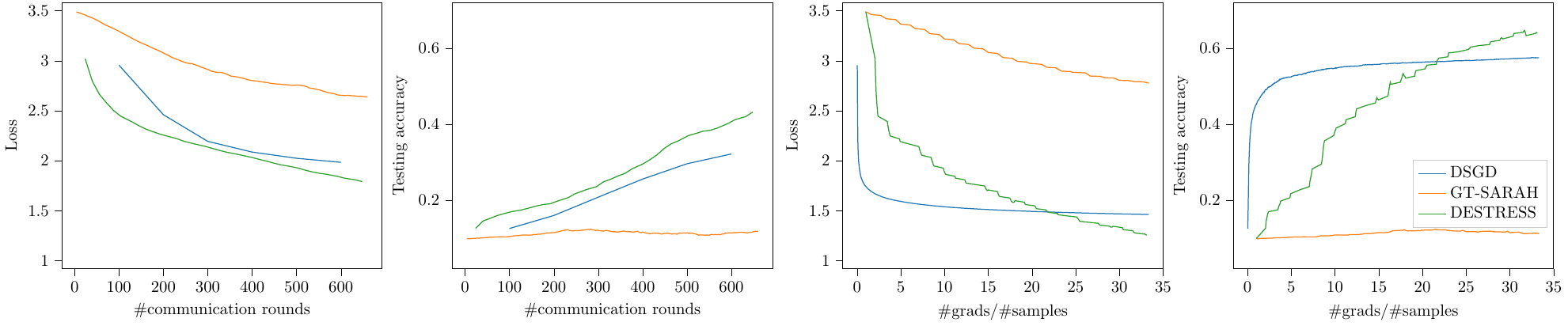}
        (c)  Path graph 
    \caption{The training loss and testing accuracy with respect to the number of communication rounds (left two panels) and gradient evaluations (right two panels) for DSGD, GT-SARAH and \myalg when training a one-hidden-layer neural network on the MNIST dataset.
        Due to the initial full-gradient computation, the gradient evaluations of \myalg and GT-SARAH do not start from $0$.
        \label{fig:nn}
    }
\end{figure}

\subsection{Neural network training}%
\label{sub:neural_network}

Next, we compare the performance of \myalg with comparisons to DSGD and GT-SARAH for training a one-hidden-layer neural network with $64$ hidden neurons and sigmoid activations for classifying the MNIST dataset \cite{deng2012mnist}.
We evenly split $60,000$ training samples to $20$ agents at random. \Cref{fig:nn} plots the training loss and testing accuracy against the number of communication rounds and gradient evaluations for all algorithms. Again,
\myalg significantly outperforms other algorithms in terms of computation and communication costs due to the larger step size and extra mixing, which validates our theoretical analysis.

\section{Conclusions}%
\label{sec:conclusion}

In this paper,
we proposed \myalg for decentralized nonconvex finite-sum optimization, where both its theoretical convergence guarantees and empirical performances on real-world datasets were presented. In sum, \myalg matches the optimal IFO complexity of centralized SARAH-type methods for finding first-order stationary points,
and improves both computation and communication complexities for a broad range of parameters regimes compared with existing approaches.
A natural and important extension of this paper is to generalize and develop convergence guarantees of \myalg for finding second-order stationary points, which we leave to future works.
 
\section*{Acknowledgements}
This work is supported in part by ONR N00014-19-1-2404, by AFRL under FA8750-20-2-0504, and by NSF under CCF-1901199 and CCF-2007911. The authors thank Ran Xin for helpful discussions.

\bibliographystyle{alphaabbr}
\newcommand{\etalchar}[1]{$^{#1}$}

\appendix

\section{Experiment details}
\label{sec:baseline}
For completeness, we list two baseline algorithms, DSGD \cite{nedic2009distributed,lian2017can} (cf. \Cref{alg:dsgd}) and
GT-SARAH \cite{xin2020near} (cf. \Cref{alg:dget}), which are compared numerically against the proposed \myalg algorithm in Section~\ref{sec:numerical}. Furthermore, the detailed hyperparameter settings for the experiments in Section~\ref{sub:logistic_regression} and Section~\ref{sub:neural_network} are listed in Table~\ref{table:params_lr} and Table~\ref{table:params_nn}, respectively.

\begin{figure}[htb]
\centering
\begin{minipage}{\linewidth}
\begin{algorithm}[H]
\caption{Decentralized stochastic gradient descent (DSGD)}
\label{alg:dsgd}
\begin{algorithmic}[1]
    \STATE {\textbf{input:} initial parameter $\bbx^{(0)}$,
        step size schedule $\eta_t$,
        number of outer loops $T$.
    }

    \STATE {\textbf{initialization:} set
    $\x_i^{(0)} = \bbx^{(0)}$.}

    \FOR {$t = 1, \ldots, T$}
        \STATE {Each agent $i$ samples a data point $\bz_i^{(t)}$ from $\cM_i$ uniformly at random and compute the stochastic gradient:}
$$\bg_i^{(t)} = \nabla \ell (\bu_i^{(t)}; \bz_i^{(t)}).$$

        \STATE{Update via local communication:
        $\x^{(t+1)} = \W ( \x^{(t)} - \eta_t \bg^{(t)} )$.}

    \ENDFOR
    \STATE {\textbf{output:} $\xout = \bbx^{(T)}$.}	
\end{algorithmic}
\end{algorithm}

\end{minipage}
\end{figure}

\begin{figure}[htb]
\centering
\begin{minipage}{\linewidth}

\begin{algorithm}[H]
\caption{GT-SARAH}
\label{alg:dget}
\begin{algorithmic}[1]
    \STATE {\textbf{input:} initial parameter $\bbx^{(0)}$,
        step size $\eta$, 
        number of outer loops $T$,
        number of inner loops $q$.
    }

    \STATE {\textbf{initialization:} set
        $\bv^{(0)} = \y^{(0)} = \nabla F(\x^{(0)})$.}

    \FOR {$t = 1, \ldots, T$}

        \STATE{Update via local communication $\x^{(t)} = \W \x^{(t-1)} - \eta \y^{(t-1)}$.}
        \IF{$\mod (t, q) = 0$}
            \STATE{$\bv^{(t)} = \nabla F(\x^{(t)}).$}

        \ELSE
        \STATE {Each agent $i$ samples a mini-batch $\cZ_i^{(t)}$ from $\cM_i$ uniformly at random, and then performs the following updates:}
 $$\bv_i^{(t)} = \frac1b \sum_{\bz_i \in \cZ_i^{(t)}} \big( \nabla \ell (\x_i^{(t)}; \bz_i) - \nabla \ell (\x_i^{(t-1)}; \bz_i) \big) + \bv_i^{(t-1)}.$$ 

        \ENDIF
        \STATE{Update via local communication $\y^{(t)} = \W \y^{(t-1)} + \bv^{(t)} - \bv^{(t-1)}$.}

    \ENDFOR
    \STATE {\textbf{output:} $\xout = \bbx^{(T)}$.}	
\end{algorithmic}
\end{algorithm}

\end{minipage}
\end{figure}

\begin{table}[ht]
    \centering
    
\begin{tabular}{c|c|cccccc|ccc}
\toprule
    Algorithms
    & DSGD
    & \multicolumn{6}{c|}{DESTRESS}
    & \multicolumn{3}{c}{GT-SARAH}     \\
    \midrule 

Parameters
    & $\eta_0$
    & $\eta$
    & $p$
    & $\KI$
    & $\KO$
    & $b$
    & $S$
    & $\eta$
    & $b$
    & $S$
    \\
    \hline
    \hline

   Erd\"{o}s-R\`{e}nyi
     & $1$
    & $1$
    & $1$
    & $2$
    & $2$
    & $10$
    & $10$
    & $0.1$
    & $10$
    & $10$
    \\
    
    \hline
    
    Grid
     & $1$
    & $1$
    & $1$
    & $2$
    & $2$
    & $10$
    & $10$
    & $0.01$
    & $10$
    & $10$
    \\

    \hline

    Path
     & $1$
    & $1$
    & $1$
    & $8$
    & $8$
    & $10$
    & $10$
    & $0.001$
    & $10$
    & $10$
    \\
\bottomrule
\end{tabular}
    \caption{Parameter settings for the experiments on regularized logistic regression in Section~\ref{sub:logistic_regression}.
    \label{table:params_lr}
    }
\end{table}

\begin{table}[!ht]
    \centering
    
\begin{tabular}{c|c|cccccc|ccc}
\toprule
    Algorithms
    & DSGD
    & \multicolumn{6}{c|}{DESTRESS}
    & \multicolumn{3}{c}{GT-SARAH}     \\
    \hline

Parameters
    & $\eta_0$
    & $\eta$
    & $p$
    & $\KI$
    & $\KO$
    & $b$
    & $S$
    & $\eta$
    & $b$
    & $S$
    \\
    \hline
    \hline

   Erd\"{o}s-R\`{e}nyi
    & $1$
    & $0.1$
    & $1$
    & $2$
    & $2$
    & $100$
    & $10$
    & $0.01$
    & $100$
    & $10$
    \\
    
    \hline
    
    Grid
    & $1$
    & $0.1$
    & $1$
    & $2$
    & $2$
    & $100$
    & $10$
    & $0.001$
    & $100$
    & $10$
    \\

    \hline

    Path
    & $1$
    & $0.1$
    & $1$
    & $8$
    & $8$
    & $100$
    & $10$
    & $0.001$
    & $100$
    & $10$
    \\
\bottomrule
\end{tabular}
    \caption{Parameter settings for the experiments on neural network training in Section~\ref{sub:neural_network}.
    \label{table:params_nn}
    }
\end{table}

\section{Proof of Theorem \ref{theorem:network_sarah_non_convex}}%
\label{sec:proof_of_theorem_ref}

For notation simplicity, let 
\begin{align*}
\alphai &= \alpha^{\KI}, \qquad\alphao = \alpha^{\KO}
\end{align*} 
throughout the proof. 
In addition, with a slight abuse of notation, we define the global gradient $\nabla f(\x) \in \R^{nd}$ of an $(nd)$-dimensional vector $\x=\big[ \x_1^\top, \cdots, \x_n^\top \big]^{\top}$, where $\x_i\in\mathbb{R}^d$, as follows
\begin{align}
	\label{eq:glo_gradient}
    \nabla f(\x) := [\nabla f(\x_1)^\top, \cdots, \nabla f(\x_n)^\top]^\top .
\end{align}

The following fact is a straightforward consequence of our assumption on the mixing matrix $\bW $ in \Cref{definition:mixing_matrix}.

\begin{fact}\label{lemma:prelim}
Let 
$ \x = \big[ \x_1^\top, \cdots, \x_n^\top \big]^{\top}$,
     and
  $ \bbx = \frac1n \sum_{i=1}^n \x_i$, where $\x_i\in\mathbb{R}^d$. For a mixing matrix $\bW \in \mathbb{R}^{n\times n}$ satisfying \Cref{definition:mixing_matrix}, we have
  \begin{enumerate}
  \item $\ave(\bW \otimes \bI_d) \x = \ave \x = \bbx $;
  \item $\big( \bI_{nd} - \mean \big) (\bW \otimes \bI_d) =(\bW \otimes \bI_d -\mean) \big( \bI_{nd} - \mean \big)$.
  \end{enumerate}
\end{fact}

To begin with, we introduce a key lemma
that upper bounds the norm of the gradient of the global loss function evaluated at the average local estimates over $n$ agents,
in terms of the function value difference at the beginning and the end of the inner loop,
the gradient estimation error, and the norm of gradient estimates.
 
\begin{lemma}[Inner loop induction] 
    \label{lemma:outer_loop_induction}
    Assume \Cref{assumption:lipschitz_gradient} holds.
    After $S \geq 1$ inner loops, one has
    \begin{align*}
        \sum_{s=0}^{S-1} \| \nabla f(\bbu^{(t), s})  \|^2_2
        & \leq \frac2\eta \Big( f(\bbu^{(t), 0}) - f(\bbu^{(t), S}) \Big) \\
        &~~~~~~~~+ \sum_{s=0}^{S-1} \big\| \nabla f(\bbu^{(t), s}) - \bbv^{(t), s} \big\|^2_2
        - (1 - \eta L) \sum_{s=0}^{S-1} \big\| \bbv^{(t), s} \big\|^2_2.
    \end{align*}
\end{lemma}
\begin{proof}[Proof of \Cref{lemma:outer_loop_induction}]
The local update rule \eqref{eq:inner_loop_var}, combined with Lemma~\ref{lemma:prelim}, yields 
$$ \bbu^{(t), s+1} = \bbu^{(t), s} - \eta \bbv^{(t), s}.$$
By \Cref{assumption:lipschitz_gradient}, we have
\begin{align}
    f(\bbu^{(t), s+1})
    &= f(\bbu^{(t), s} - \eta \bbv^{(t), s}) \notag \\
    &\leq f(\bbu^{(t), s}) - \big\langle \nabla f(\bbu^{(t), s}), \eta \bbv^{(t), s} \big\rangle + \frac L2 \big\| \eta \bbv^{(t), s} \big\|^2_2 \notag \\
    &= f(\bbu^{(t), s})
    - \frac\eta2 \big\| \nabla f(\bbu^{(t), s}) \big\|^2_2
    + \frac{\eta}{2} \big\| \nabla f(\bbu^{(t), s}) - \bbv^{(t), s} \big\|^2_2  
     - \Big(\frac\eta2 - \frac {\eta^2 L}{2}\Big) \big\| \bbv^{(t), s} \big\|^2_2,
    \label{eq:network_sarah_non_convex_lemma_1}
\end{align}
where the last equality is obtained by applying $-\langle \ba, \bb \rangle = \frac12 \big( \| \ba - \bb \|^2_2 - \| \ba \|^2_2 - \| \bb \|^2_2 \big)$. Summing over $s=0,\ldots, S-1$ finishes the proof.
\end{proof}

    Because the output $\xout$ is chosen from $\big\{\bu_i^{(t), s-1} | i \in [n], t \in [T], s \in [S] \big\}$ uniformly at random,
we can compute the expectation of the output's gradient as follows:
\begin{align}
    nTS \E \big\| \nabla f(\xout) \big\|_2^2
    &= \sum_{i=1}^n \sum_{t=1}^T \sum_{s=0}^{S-1} \E \big\| \nabla f(\bu_i^{(t), s}) \big\|_2^2 \notag \\
 &  \overset{\text{(i)}}{=} \sum_{t=1}^T \sum_{s=0}^{S-1} \E \big\| \nabla f(\bu^{(t), s}) \big\|_2^2 \notag \\
 &   = \sum_{t=1}^T \sum_{s=0}^{S-1} \E \big\| \nabla f(\bu^{(t), s}) - \nabla f(\onet\bbu^{(t), s}) + \nabla f(\onet\bbu^{(t), s}) \big\|_2^2 \notag \\
    & \overset{\text{(ii)}}{\leq} 2 \sum_{t=1}^T \sum_{s=0}^{S-1} \Big( \E \big\| \nabla f(\bu^{(t), s}) - \nabla f(\onet\bbu^{(t), s}) \big\|_2^2 + \E \big\| \nabla f(\onet\bbu^{(t), s}) \big\|_2^2 \Big) \notag \\
   &   \overset{\text{(iii)}}{\leq} 2 \sum_{t=0}^T \sum_{s=0}^{S-1} \Big( L^2 \E \big\| \bu^{(t), s} - \onet\bbu^{(t), s} \big\|_2^2 + n \E \big\| \nabla f(\bbu^{(t), s}) \big\|_2^2 \Big) ,
    \label{eq:bound_ouf_output_gradient_1}
\end{align}
where (i) follows from the change of notation using \eqref{eq:glo_gradient}, (ii) follows from the Cauchy-Schwartz inequality, and (iii) follows from \Cref{assumption:lipschitz_gradient} and extending the summation to $t = 0, ... T$. Then, in view of \Cref{lemma:outer_loop_induction},
\eqref{eq:bound_ouf_output_gradient_1} can be further bounded by
\begin{align}
    n TS \E \big\| \nabla f(\xout) \big\|_2^2
   & \leq \frac{4n}{\eta} \Big( \E[ f(\bbx^{(0)})] - f^* \Big)
    + 2 L^2 \sum_{t=0}^T \sum_{s=0}^{S-1} \E \big\| \bu^{(t), s} - \onet\bbu^{(t), s} \big\|_2^2 
    \notag \\
    & \qquad+ 2n \sum_{t=0}^T \sum_{s=0}^{S-1} \Big( \E \big\| \nabla f(\bbu^{(t), s}) - \bbv^{(t), s} \big\|^2_2
    - (1 - \eta L) \E \big\| \bbv^{(t), s} \big\|^2_2 \Big) ,
    \label{eq:bound_of_output_gradient_2}
\end{align}
where we use $\bbu^{(t), 0} = \bbx^{(t)}$ and $f(\bbu^{(t), S}) \geq f^*$.

Next, we present \Cref{lemma:sum_of_inner_loop_errors,lemma:sum_of_outer_error} to bound the double sum in \eqref{eq:bound_of_output_gradient_2},
whose proofs can be found in \Cref{sec:proof_of_lemma:sum_of_inner_loop_errors} and \Cref{sub:proof_of_lemma:sum_of_outer_vars}, respectively.

\begin{lemma}[Sum of inner loop errors]
\label{lemma:sum_of_inner_loop_errors}
Assuming all conditions in \Cref{theorem:network_sarah_non_convex} hold.
For all $t > 0$, we can bound the summation of inner loop errors as
\begin{align*}
    2 L^2 \sum_{s=0}^{S-1} \E \big\| & \bu^{(t), s} - \onet\bbu^{(t), s} \big\|_2^2
    + 2n \sum_{s=0}^{S-1} \E \big\| \nabla f(\bbu^{(t), s}) - \bbv^{(t), s} \big\|^2_2 \\
    & \leq \frac{64 L^2}{1 - \alphai} \cdot \Big( \frac{S}{npb} + 1 \Big) \E \big\| \bx^{(t)} - \onet\bbx^{(t)} \big\|_2^2 + 2 \alphai^2 \E \big\| \bs^{(t)} - \onet\bbs^{(t)} \big\|_2^2
    + \frac{2n}{25} \sum_{s=1}^{S} \E \big\| \bbv^{(t), s-1} \big\|^2_2 .
\end{align*}

\end{lemma}

\begin{lemma}[Sum of outer loop gradient estimation error and consensus error]
\label{lemma:sum_of_outer_error} 
Assuming all conditions in \Cref{theorem:network_sarah_non_convex} hold.
We have
\begin{align*}
    \frac{64 L^2}{1 - \alphai} \cdot \Big( \frac{S}{npb} + 1 \Big) \sum_{t=0}^T \E \big\| \bx^{(t)} - \onet\bbx^{(t)} \big\|_2^2
    & + 2 \alphai^2 \sum_{t=0}^T \E \big\| \bs^{(t)} - \onet\bbs^{(t)} \big\|_2^2
     \leq \frac{11n}{25} \sum_{t=1}^T \sum_{s=0}^{S-1} \E \big\| \bbv^{(t), s} \big\|^2_2 .
\end{align*}
\end{lemma}

Using \Cref{lemma:sum_of_inner_loop_errors},
\eqref{eq:bound_of_output_gradient_2} can be bounded as follows:
\begin{align}
    n TS \E \big\| &\nabla f(\xout) \big\|_2^2
     < \frac{4n}{\eta} \Big( \E[ f(\bbx^{(t), 0})] - f^* \Big)
    - 2n \Big(\frac{24}{25} - \eta L \Big) \sum_{t=0}^T \sum_{s=1}^{S} \E \big\| \bbv^{(t), s-1} \big\|^2_2 \notag \\
    & + \frac{64 L^2}{1 - \alphai} \cdot \Big( \frac{S}{npb} + 1 \Big) \sum_{t=0}^T \E \big\| \bx^{(t)} - \onet\bbx^{(t)} \big\|_2^2 + 2 \alphai^2 \sum_{t=0}^T \E \big\| \bs^{(t)} - \onet\bbs^{(t)} \big\|_2^2 ,
    \label{eq:proof_of_theorem_2}
\end{align}
where we bound the sum of inner loop errors
$L^2 \sum_{s=0}^{S-1} \E \big\| \bu^{(t), s} - \onet\bbu^{(t), s} \big\|_2^2$
and $n \sum_{s=0}^{S-1} \E \big\| \nabla f(\bbu^{(t), s}) - \bbv^{(t), s} \big\|^2_2$
by the initial value of each inner loop $\E \big\| \bx^{(t)} - \onet\bbx^{(t)} \big\|_2^2$ and $\E \big\| \bs^{(t)} - \onet\bbs^{(t)} \big\|_2^2$,
and the summation of the norm of average inner loop gradient estimator $n \sum_{s=1}^{S} \E \big\| \bbv^{(t), s-1} \big\|^2_2$.

By \Cref{lemma:sum_of_outer_error},
\eqref{eq:proof_of_theorem_2} can be further bounded as
\begin{align}
    n TS \E \big\| \nabla f(\xout) \big\|_2^2
    &\leq \frac{4n}{\eta} \Big( \E[ f(\bbx^{(t), 0})] - f^* \Big)
    - 2 n \Big(\frac{37}{50} - \eta L \Big) \sum_{t=1}^T \sum_{s=0}^{S-1} \E \big\| \bbv^{(t), s} \big\|^2_2 \notag \\
    &< \frac{4n}{\eta} \Big( \E[ f(\bbx^{(t), 0})] - f^* \Big) , \notag
\end{align}
which concludes the proof.
\section{Proof of Corollary \ref{corollary:network_sarah_non_convex_two}}
\label{sec:proof_of_corollay_1}

Without loss of generality, we assume $n \geq 2$. Otherwise, the problem reduces to the centralized setting with a single agent $n=1$, and the bound holds trivially.
We will confirm the choice of parameters in \Cref{corollary:network_sarah_non_convex_two} in the following paragraphs,
and finally obtain the IFO complexity and communication complexity.

\paragraph{Step size $\eta$.}%
We first assume $\alphai \leq \frac p2 \leq \frac12$ and $\alphao \leq \frac{1}{\sqrt{npb} + 1} \leq \frac12$,
which will be proved to hold shortly,
then we can verify the step size choice meets the requirement in \eqref{eq:step_size_condition} as:
\begin{align*}
     \frac{(1 - \alphai)^3 (1 - \alphao)}{1 + \alpha^{\KI} \alpha^{\KO} \sqrt{p n b}} \cdot \frac{1}{10 L \big( \sqrt{S/(npb)} + 1 \big)}
    & \geq \frac{(1/2)^{4}}{2} \cdot \frac{1}{20 L}
    = \frac{1}{640 L}.
\end{align*}

\paragraph{Mixing steps $\KI$ and $\KO$.}%
\label{par:mixing_steps_ki_and_ko_}

Using Chebyshev's acceleration \cite{arioli2014chebyshev} to implement the mixing steps, it amounts to an improved mixing rate of $ \alpha_{\mathsf{cheb}} \asymp 1 - \sqrt{2(1 - \alpha)}$, when the original mixing rate $\alpha$ is close to $1$.
Set $\KI = \left\lceil \frac{\log (2 / p)}{\sqrt{1 - \alpha}} \right\rceil$ and $\KO =\left\lceil \frac{\log ( \sqrt{npb} + 1)}{\sqrt{1 - \alpha}} \right\rceil$. We are now positioned to examine the effective mixing rate
$\alphai = \alpha_{\mathsf{cheb}}^{\KI}$ and $\alphao = \alpha_{\mathsf{cheb}}^{\KO}$, as follows
\begin{equation*}
\alphao = \alpha_{\mathsf{cheb}}^{\KO}
\overset{\mathrm{(i)}}{\leq}
\alpha_{\mathsf{cheb}}^{   \frac{\log ( \sqrt{npb} + 1)}{\sqrt{1 - \alpha}}  }
\asymp \alpha_{\mathsf{cheb}}^{   \frac{\sqrt 2 \log ( \sqrt{npb} + 1)}{1 - \alpha_{\mathsf{cheb}}}  }
\overset{\mathrm{(ii)}}{\leq} \alpha_{\mathsf{cheb}}^{\frac{\sqrt 2 \log ( \sqrt{npb} + 1)}{- \log\alpha_{\mathsf{cheb}}}}
< \frac{1}{\sqrt{npb} + 1} \overset{\mathrm{(iii)}}{\leq} \frac{1}{2},
\end{equation*}
where (i) follows from $\KO =\left\lceil \frac{\log ( \sqrt{npb} + 1)}{\sqrt{1 - \alpha}} \right\rceil$,
(ii) follows from $\log x \leq x - 1$, $\forall x > 0$,
and (iii) follows from $n \geq 1$ and $b \geq 1$.
By a similar argument, we have $\alphai =\alpha_{\mathsf{cheb}}^{\KI} \leq \frac p2$.

\paragraph{Complexity.}
Plugging in the selected parameters into \eqref{eq:theorem_1} in \Cref{theorem:network_sarah_non_convex},
We have
\begin{align*}
    \E \big\| \nabla f(\xout) \big\|_2^2
    \leq& \frac{4}{\eta TS} \Big( \E[ f(\bbx^{(t), 0})] - f^* \Big)
    = O\Big( \frac{L}{T \sqrt{mn}} \Big) .
\end{align*}
Consequently, the outer iteration complexity is 
$T = O\Big(1 + \frac {L}{(mn)^{1/2}\epsilon} \Big).$
With this in place, we summarize the communication and IFO complexities as follows:
\begin{itemize}
    \item The communication complexity is $T \cdot (S \KI + \KO)
        = O\Big( \frac{(mn)^{1/2} \log \big( 2 (n/m)^{1/2} + 2 \big) + \log \big( (mn)^{1/4} + 1 \big)}{\sqrt{1 - \alpha}} \cdot \Big( 1 + \frac {L}{(mn)^{1/2}\epsilon} \Big) \Big)
        = O\Big(\frac{\log \big((n/m)^{1/2} + 2 \big)}{\sqrt{1 - \alpha}} \cdot \big( (mn)^{1/2} + \frac {L}{\epsilon} \big) \Big)$,
        where we use
        $
        2/p
        = \frac{2 \left\lceil \sqrt{m/n} \right\rceil}{\sqrt{m / n}}
        \leq \frac{2 ( \sqrt{m/n} + 1)}{\sqrt{m / n}}
        = 2 (\sqrt{n / m} + 1)
        $ to bound $\KI$.

    \item The IFO complexity is $T \cdot (Spb + 2m)  = O\Big( m + \frac{(m/n)^{1/2} L}{\epsilon} \Big)$.
\end{itemize}

\section{Proof of Lemma \ref{lemma:sum_of_inner_loop_errors}}%
\label{sec:proof_of_lemma:sum_of_inner_loop_errors}

This section proves \Cref{lemma:sum_of_inner_loop_errors}.
\Cref{sub:sum_of_inner_gradients,sub:sum_of_inner_consensus}
bounds the expected inner loop gradient estimation error and consensus errors by their previous values and the sum of inner loop gradient estimator's norms,
\Cref{sub:inner_linear_system} then creates a linear system to compute the summation of inner loop errors using their initial values of each inner loop,
which concludes the proof.

\subsection{Sum of inner loop gradient estimation errors}%
\label{sub:sum_of_inner_gradients}

To begin with, note that the gradient estimation error at the $s$-th inner loop iteration can be written as
\begin{align}
    &~~~~\E \big\| \nabla f(\bbu^{(t), s}) - \bbv^{(t), s} \big\|_2^2 \notag \\
    &= \E \Big\| \ave \Big( \nabla F(\onet\bbu^{(t), s}) - \bv^{(t), s} \Big) \Big\|_2^2 \notag \\
  &  = \E \Big\| \ave \big( \nabla F(\onet\bbu^{(t), s}) - \nabla F(\bu^{(t), s}) \big)
    + \ave \big( \nabla F(\bu^{(t), s}) - \bv^{(t), s} \big) \Big\|_2^2 \notag \\
 &   \leq 2 \E \Big\| \ave \big( \nabla F(\onet\bbu^{(t), s}) - \nabla F(\bu^{(t), s}) \big) \Big\|_2^2
    + 2 \E \Big\| \ave \big( \nabla F(\bu^{(t), s}) - \bv^{(t), s} \big) \Big\|_2^2 \notag \\
 &   \leq \frac{2 L^2}{n} \E \big\| \bu^{(t), s} - \onet\bbu^{(t), s} \big\|_2^2
    + 2 \E \Big\| \ave \big( \nabla F(\bu^{(t), s}) - \bv^{(t), s} \big) \Big\|_2^2,
    \label{eq:sum_of_inner_gradients_1}
\end{align}
where the first equality follows from \eqref{eq:def_gradient}, and the last inequality is due to \Cref{assumption:lipschitz_gradient}. To continue, 
the expectation of the second term in \eqref{eq:sum_of_inner_gradients_1} can be bounded as
\begin{align}
    &~~~~ \E \Big\| \ave \Big( \nabla F(\bu^{(t), s}) - \bv^{(t), s} \Big) \Big\|_2^2 \notag \\
    & = \E \Big\| \ave \Big( \big( \nabla F(\bu^{(t), s}) - \bv^{(t), s} \big) - \big( \nabla F(\bu^{(t), s-1}) - \bv^{(t), s-1} \big) + \big( \nabla F(\bu^{(t), s-1}) - \bv^{(t), s-1} \big) \Big) \Big\|_2^2 \notag \\
    & \overset{\text{(i)}}{=} \E \Big\| \ave \Big( \big( \nabla F(\bu^{(t), s}) - \bv^{(t), s} \big) - \big( \nabla F(\bu^{(t), s-1}) - \bv^{(t), s-1} \big) \Big) \Big\|_2^2 \notag \\
    &\qquad + \E \Big\| \ave \big( \nabla F(\bu^{(t), s-1}) - \bv^{(t), s-1} \big) \Big\|_2^2 \notag \\
    & \overset{\text{(ii)}}{=} \sum_{k=1}^s \E \Big\| \ave \Big( \big( \nabla F(\bu^{(t), k}) - \bv^{(t), k} \big) - \big( \nabla F(\bu^{(t), k-1}) - \bv^{(t), k-1} \big) \Big) \Big\|_2^2 \notag \\
    & \qquad + \E \Big\| \ave \big( \nabla F(\bu^{(t), 0}) - \bv^{(t), 0} \big) \Big\|_2^2  \notag \\
    & \overset{\text{(iii)}}{=} \sum_{k=1}^s \E \Big\| \ave \Big( \big( \nabla F(\bu^{(t), k}) - \bv^{(t), k} \big) - \big( \nabla F(\bu^{(t), k-1}) - \bv^{(t), k-1} \big) \Big) \Big\|_2^2 .
    \label{eq:sum_of_inner_gradients_2}
\end{align}
Here, (i) follows from the expectation with respect to the activating indicator $\lambda_i^{(t), s}$ and random samples $\mathcal{Z}^{(t),s}$, conditioned on $\bu^{(t), s-1}$ and $\bv^{(t), s-1}$:
\begin{align}
    &~~~~\E \left[  \ave \big( \nabla F(\bu^{(t), s}) - \bv^{(t), s} \big) \Big| \bu^{(t), s-1},\bv^{(t), s-1} \right] \notag \\
    &= \frac1n \sum_{i=1}^n \nabla f_i(\bu_i^{(t), s}) - \E \Bigg[ \frac{1}{npb} \sum_i \lambda_i^{(t), s} \sum_{\bz_i \in \cZ_i^{(t), s}} \Big( \nabla \ell (\bu_i^{(t), s}; \bz_i) 
                                  - \nabla \ell (\bu_i^{(t), s-1}; \bz_i) \Big) \Big| \bu^{(t), s-1},\bv^{(t), s-1} \Bigg] - \bbv^{(t), s-1} \notag \\
    &= \frac1n \sum_{i=1}^n \nabla f_i(\bu_i^{(t), s}) - \frac{1}{np} \sum_i \E \big[\lambda_i^{(t), s} \big] \Big( \nabla f_i(\bu_i^{(t), s}) - \nabla f_i(\bu_i^{(t), s-1}) \Big) - \bbv^{(t), s-1} \notag \\
    &=   \frac1n \sum_{i=1}^n \nabla f_i(\bu_i^{(t), s-1}) - \bbv^{(t), s-1} \notag \\
    &=  \ave \big( \nabla F(\bu^{(t), s-1}) - \bv^{(t), s-1} \big), \label{eq:zero_cross_term}
\end{align}
(ii) follows by recursively applying the relation obtained from (i); and (iii) follows from the property of gradient tracking, i.e.
\begin{align} \label{eq:gradient_tracking_property}
\ave  \nabla F(\bu^{(t), 0})  &   = \frac1n \sum_{i=1}^n \nabla f_i(\bu_i^{(t), 0})= \frac1n \sum_{i=1}^n \nabla f_i(\x_i^{(t)}) = \bbs^{(t)} =\bbv^{(t), 0} ,
\end{align}
which leads to $ \ave \big( \nabla F(\bu^{(t), 0}) - \bv^{(t), 0} \big)  = \bm0$.

We now continue to bound each term in \eqref{eq:sum_of_inner_gradients_2}, which can be viewed as the variance of the stochastic gradient, as
\begin{align}
    &~~~~ \E \Big\| \ave \Big( \big( \nabla F(\bu^{(t), s}) - \bv^{(t), s} \big) - \big( \nabla F(\bu^{(t), s-1}) - \bv^{(t), s-1} \big) \Big) \Big\|_2^2  \notag \\
    & \overset{\text{(i)}}{=} \E \Bigg\| \frac{1}{nb} \sum_{i=1}^n \sum_{\bz_i \in \cZ_i^{(t), s}}  \Big( \big( \nabla f_i(\bu_i^{(t), s}) - \nabla f_i(\bu_i^{(t), s-1}) \big) - \frac{\lambda_i^{(t), s}}{p} \big( \nabla \ell(\bu_i^{(t), s}; \bz_i) - \nabla \ell(\bu_i^{(t), s-1}; \bz_i) \big) \Big) \Bigg\|_2^2 \notag \\
    & \overset{\text{(ii)}}{=} \frac{1}{n^2 b^2}  \sum_{i=1}^n \sum_{\bz_i \in \cZ_i^{(t), s}} \E \Big\| \big( \nabla f_i(\bu_i^{(t), s}) - \nabla f_i(\bu_i^{(t), s-1}) \big) - \frac{\lambda_i^{(t), s}}{p} \big( \nabla \ell(\bu_i^{(t), s}; \bz_i) - \nabla \ell(\bu_i^{(t), s-1}; \bz_i) \big) \Big\|_2^2 \notag \\
    & \overset{\text{(iii)}}{=} \frac{1}{n^2 p^2 b^2}  \sum_{i=1}^n \sum_{\bz_i \in \cZ_i^{(t), s}}  \E \Big[ \big(\lambda_i^{(t), s} \big)^2 \Big] \E \big\| \nabla \ell(\bu_i^{(t), s}; \bz_i) - \nabla \ell(\bu_i^{(t), s-1}; \bz_i) \big\|_2^2
    - \frac{1}{n^2 b} \E \big\| \nabla F(\bu^{(t), s}) - \nabla F(\bu^{(t), s-1}) \big\|_2^2 \notag \\
    & \leq \frac{L^2}{n^2 pb} \E \big\| \bu^{(t), s} - \bu^{(t), s-1} \big\|_2^2 , \label{eq:var_stoc_grad}
\end{align}
where (i) follows from the update rules \eqref{eq:inner_loop_sg} and \eqref{eq:inner_loop_grad}, (ii) follows from the independence of samples and $\E \big[ \lambda_i^{(t), s} \big] = p$, (iii) follows from similar argument with \eqref{eq:zero_cross_term}, and the last inequality follows from \Cref{assumption:lipschitz_gradient} and $\E \Big[ \big(\lambda_i^{(t), s} \big)^2 \Big] = p$. 

In view of \eqref{eq:inner_loop_var}, the difference between inner loop variables in \eqref{eq:var_stoc_grad} can be bounded deterministically as
\begin{align}
    &~~~~ \big\| \bu^{(t), s} - \bu^{(t), s-1} \big\|_2^2 \notag \\
    &= \big\| \WKI ( \bu^{(t), s-1} - \eta \bv^{(t), s-1} ) - \bu^{(t), s-1} \big\|_2^2 \notag \\
    & \overset{\text{(i)}}{=} \Big\| \Big( \WKI - \bI_{nd} \Big) ( \bu^{(t), s-1} - \onet \bbu^{(t), s-1} ) \notag \\
    &  \qquad - \eta \Big( \WKI - \mean \Big) (\bv^{(t), s-1} - \onet\bbv^{(t), s-1} )
    - \eta \onet\bbv^{(t), s-1} \Big\|_2^2 \notag \\
    & \overset{\text{(ii)}}{=} \Big\| \Big( \WKI - \bI_{nd} \Big) ( \bu^{(t), s-1} - \onet \bbu^{(t), s-1} ) \notag \\
    &  \qquad  - \eta \Big( \WKI - \mean \Big) (\bv^{(t), s-1} - \onet\bbv^{(t), s-1} ) \Big\|_2^2
    + \eta^2 n \big\| \bbv^{(t), s-1} \big\|_2^2 \notag \\
    & \leq 8 \big\| \bu^{(t), s-1} - \onet \bbu^{(t), s-1} \|_2^2
    + 2 \alphai^2 \eta^2 \big\| \bv^{(t), s-1} - \onet\bbv^{(t), s-1} \big\|_2^2
    + \eta^2 n \big\| \bbv^{(t), s-1} \big\|_2^2 ,
    \label{eq:sum_of_inner_gradients_10}
\end{align}
where (i) and (ii) follow from $\big(\WKI - \bI_{nd} \big) ( \onet \bbx ) = \bm{0} $ and $\big( \WKI - \mean \big) (\onet\bbx ) = \bm{0} $ for any mean vector $\bbx$;  and the last inequality follows from the property of the mixing matrix $\big\| \WKI- \bI_{nd} \big\|_{\mathsf{op}} \leq 2$ and $\big\| \WKI - \mean \big\|_{\mathsf{op}} \leq \alphai$.

Plugging \eqref{eq:var_stoc_grad} and \eqref{eq:sum_of_inner_gradients_10} into 
\eqref{eq:sum_of_inner_gradients_2}, we can further obtain 
\begin{align*}
    &~~~~ \E \Big\| \Big( \mean \Big) \Big( \nabla F(\bu^{(t), s}) - \bv^{(t), s} \Big) \Big\|_2^2 \\
    &\leq \frac{L^2}{n^2 pb} \sum_{k=1}^s \E \big\| \bu^{(t), k} - \bu^{(t), k-1} \big\|^2_2 \\
    &\leq \frac{8 L^2}{n^2 pb} \sum_{k=0}^{s-1} \E \big\| \bu^{(t), k} - \onet \bbu^{(t), k} \|_2^2
     + \frac{2 \alphai^2 \eta^2 L^2}{n^2 pb} \sum_{k=0}^{s-1} \E \big\| \bv^{(t), k} - \onet\bbv^{(t), k} \big\|_2^2
    + \frac{\eta^2 L^2}{npb} \sum_{k=0}^{s-1} \E \big\| \bbv^{(t), k} \big\|_2^2 .
\end{align*}
Using \eqref{eq:sum_of_inner_gradients_1} and the previous inequality,
we can bound the summation of inner loop gradient estimation errors as
\begin{align*}
    &~~~~ \sum_{s=0}^{S-1} \E \big\| \nabla f(\bbu^{(t), s}) - \bbv^{(t), s} \big\|_2^2 \\
    & \leq \frac{2 L^2}{n}  \sum_{s=0}^{S-1} \E \big\| \bu^{(t), s} - \onet\bbu^{(t), s} \big\|_2^2
    + 2 \sum_{s=0}^{S-1}  \E \Big\| \ave \big( \nabla F(\bu^{(t), s}) - \bv^{(t), s} \big) \Big\|_2^2 \\
    & \leq \frac{2 L^2}{n} \sum_{s=0}^{S-1} \E \big\| \bu^{(t), s} - \onet\bbu^{(t), s} \big\|_2^2
    + \frac{16 L^2}{n^2 pb} \sum_{s=0}^{S-1} \sum_{k=0}^{s-1} \E \big\| \bu^{(t), k} - \onet \bbu^{(t), k} \big\|_2^2 \\
    & \qquad + \frac{4 \alphai^2 \eta^2 L^2}{n^2 pb} \sum_{s=0}^{S-1} \sum_{k=0}^{s-1}  \E \big\| \bv^{(t), k} - \onet\bbv^{(t), k} \big\|_2^2
    + \frac{2 \eta^2 L^2}{n pb} \sum_{s=0}^{S-1} \sum_{k=0}^{s-1} \E \big\| \bbv^{(t), k} \big\|_2^2 \\
    & \leq \Big( \frac{8S}{npb} + 1 \Big) \cdot \frac{2 L^2}{n} \sum_{s=0}^{S-1} \E \big\| \bu^{(t), s} - \onet\bbu^{(t), s} \big\|_2^2 \\
    &\qquad + \frac{4 S \alphai^2 \eta^2 L^2}{n^2 pb} \sum_{s=0}^{S-1} \E \big\| \bv^{(t), s} - \onet\bbv^{(t), s} \big\|_2^2
    + \frac{2 S \eta^2 L^2}{n pb} \sum_{s=0}^{S-1} \E \big\| \bbv^{(t), s} \big\|_2^2 ,
\end{align*}
where the last inequality is obtained by relaxing the upper bound of the summation w.r.t. $k$ from $s - 1$ to $S - 1$.

The quantity of interest can be now bounded as
\begin{align}
    2 L^2 \sum_{s=0}^{S-1} \E \big\| \bu^{(t), s} - &\onet\bbu^{(t), s} \big\|_2^2
    + 2n \sum_{s=0}^{S-1} \E \big\| \nabla f(\bbu^{(t), s}) - \bbv^{(t), s} \big\|^2_2 \notag \\
    & \leq \Big( \frac{4S}{npb} + 1 \Big) \cdot 8 L^2 \sum_{s=0}^{S-1} \E \big\| \bu^{(t), s} - \onet\bbu^{(t), s} \big\|_2^2 \notag \\
    &\qquad + \frac{8 S \alphai^2 \eta^2 L^2}{npb} \sum_{s=0}^{S-1} \E \big\| \bv^{(t), s} - \onet\bbv^{(t), s} \big\|_2^2
    + \frac{4 S \eta^2 L^2}{pb} \sum_{s=0}^{S-1} \E \big\| \bbv^{(t), s} \big\|_2^2
    \label{eq:lemma:sum_of_inner_loop_errors_1}
\end{align}

\subsection{Sum of inner loop consensus errors}%
\label{sub:sum_of_inner_consensus}
Using the update rule \eqref{eq:inner_loop_var},
the variable consensus error can be expanded deterministically as follows:
\begin{align}
 \big\| \bu^{(t), s} - \onet\bbu^{(t), s} \big\|_2^2 
  &  = \Big\| \Big( \bI_{nd} - \mean \Big) \WKI ( \bu^{(t), s-1} - \eta \bv^{(t), s-1} ) \Big\|_2^2 \notag \\
  &  \overset{\text{(i)}}{\leq} \alphai^2 \Big\| \Big( \bI_{nd} - \mean \Big) ( \bu^{(t), s-1} - \eta \bv^{(t), s-1} ) \Big\|_2^2 \notag \\
   & \leq \frac{2 \alphai^2}{1 + \alphai^2} \big\| \bu^{(t), s-1} - \onet\bbu^{(t), s-1} \big\|_2^2
    + \frac{2 \alphai^2 \eta^2}{1 - \alphai^2} \big\| \bv^{(t), s-1} - \onet\bbv^{(t), s-1} \big\|_2^2 ,
    \label{eq:sum_of_inner_consensus_1}
\end{align}
where (i) follows from the fact
$$\Big( \bI_{nd} - \mean \Big) \W = \Big(\bW \otimes \bI_d - \mean \Big) \Big( \bI_{nd} - \mean \Big)$$
and the definition of the mixing rate.
The last inequality follows from the elementary inequality
$2\langle \bm{a} , \bm{b} \rangle  \leq \frac{1-\alphai^2}{1+\alphai^2} \| \bm{a}\|_2^2 + \frac{1+\alphai^2}{1-\alphai^2}  \| \bm{b}\|_2^2 $,
so that $\| \ba + \bb \|_2^2 \leq \frac{2}{1 + \alphai^2} \| \ba \|_2^2 + \frac{2}{1 - \alphai^2} \| \bb \|_2^2 $.

Furthermore, using the update rules \eqref{eq:inner_loop_sg} and \eqref{eq:inner_loop_grad}
and defining $\bLambda^{(t), s} = \frac1p \text{diag}(\lambda_1^{(t), s}, \lambda_2^{(t), s}, \hdots, \lambda_n^{(t), s}) \otimes \bI_d$,
the gradient consensus error can be similarly expanded as follows:
\begin{align}
    &~~~~ \big\| \bv^{(t), s} - \onet\bbv^{(t), s} \big\|_2^2 \notag \\
    & = \Big\| \Big( \bI_{nd} - \mean \Big) \WKI \bg^{(t),s} \Big\|_2^2 \notag \\
    & \leq \alphai^2 \Big\| \Big( \bI_{nd} - \mean \Big)\bg^{(t),s}  \Big\|_2^2 \notag \\
    & \overset{\text{(i)}}{\leq} \frac{2 \alphai^2}{1 + \alphai^2} \big\| \bv^{(t), s-1} - \onet\bbv^{(t), s-1} \big\|_2^2 \notag \\
    & \qquad +\frac{2 \alphai^2}{1 - \alphai^2} \cdot \frac1b \sum_{\bz \in \cZ^{(t), s}} \Big\| \Big( \bI_{nd} - \mean \Big) \bLambda^{(t), s} \big( \nabla \ell (\bu^{(t), s}; \bz) - \nabla \ell (\bu^{(t), s-1}; \bz ) \big)   \Big\|_2^2 \notag \\
    & \overset{\text{(ii)}}{\leq} \frac{2 \alphai^2}{1 + \alphai^2} \big\| \bv^{(t), s-1} - \onet\bbv^{(t), s-1} \big\|_2^2
    +\frac{2 \alphai^2 L^2}{(1 - \alphai^2)p^2} \big\| \bu^{(t), s} - \bu^{(t), s-1} \big\|_2^2 \notag \\
    & \overset{\text{(iii)}}{\leq} \frac{2 \alphai^2}{1 + \alphai^2} \big\| \bv^{(t), s-1} - \onet\bbv^{(t), s-1} \big\|_2^2
    +\frac{2 \alphai^2 L^2}{(1 - \alphai^2)p^2} \Big( 
        8 \big\| \bu^{(t), s-1} - \onet \bbu^{(t), s-1} \|_2^2 \notag \\
    &\qquad  + 2 \alphai^2 \eta^2 \big\| \bv^{(t), s-1} - \onet\bbv^{(t), s-1} \big\|_2^2
    + \eta^2 n \big\| \bbv^{(t), s-1} \big\|_2^2
    \Big) \notag \\
    & = \Big( \frac{2 \alphai^2}{1 + \alphai^2} + \frac{4 \alphai^4 \eta^2 L^2}{(1 - \alphai^2)p^2} \Big) \big\| \bv^{(t), s-1} - \onet\bbv^{(t), s-1} \big\|_2^2 \notag \\
    & \qquad + \frac{16 \alphai^2 L^2}{(1 - \alphai^2)p^2} \big\| \bu^{(t), s-1} - \onet\bbu^{(t), s-1} \big\|_2^2
    + \frac{2 \alphai^2 \eta^2 L^2}{(1 - \alphai^2)p^2} \cdot n \big\| \bbv^{(t), s-1} \big\|_2^2 ,
    \label{eq:sum_of_inner_consensus_2}
\end{align}
where the second term in (i) is obtained by Jensen's inequality,
(ii) follows from \Cref{assumption:lipschitz_gradient} and $\big\| \bLambda^{(t), s} \big\|_{\mathsf{op}} \leq \frac1p$,
and (iii) follows from \eqref{eq:sum_of_inner_gradients_10}.

\subsection{Linear system}%
\label{sub:inner_linear_system}

Let $\be^{(t), s} =
\begin{bmatrix}
    L^2 \E \big\| \bu^{(t), s} - \onet\bbu^{(t), s} \big\|_2^2 \\
    \E \big\| \bv^{(t), s} - \onet\bbv^{(t), s} \big\|_2^2
\end{bmatrix}$,
and
$\bb^{(t), s}
= \frac{2 \alphai^2 \eta^2 L^2}{(1 - \alphai^2)p^2}
\begin{bmatrix}
    0 \\
    n \E \big\| \bbv^{(t), s} \big\|_2^2
\end{bmatrix}
$.
By taking expectation of \eqref{eq:sum_of_inner_consensus_1} and \eqref{eq:sum_of_inner_consensus_2},
we can construct the following linear system
\begin{align}
    \be^{(t), s}
  &  \leq
    \begin{bmatrix}
        \frac{2 \alphai^2}{1 + \alphai^2} & \frac{2 \alphai^2 \eta^2 L^2}{1 - \alphai^2} \\[0.5em]
        \frac{16\alphai^2}{(1 - \alphai^2)p^2} & \frac{2 \alphai^2}{1 + \alphai^2} + \frac{4 \alphai^4 \eta^2 L^2}{(1 - \alphai^2)p^2}
    \end{bmatrix} 
    \be^{(t), s-1}
    + \bb^{(t), s-1}  \notag \\
  &  \leq
  \underbrace{  \begin{bmatrix}
        \alphai & \frac{2 \alphai^2 \eta^2 L^2}{1 - \alphai} \\[0.5em]
        \frac{16\alphai^2}{(1 - \alphai)p^2} & \alphai + \frac{4 \alphai^4 \eta^2 L^2}{(1 - \alphai)p^2}
    \end{bmatrix} }_{=:\bGI}
    \be^{(t), s-1}
    + \bb^{(t), s-1}   =
    \bGI \be^{(t), s-1} + \bb^{(t), s-1} , \label{eq:def_GI}
\end{align}
where the second inequality is due to $2 \alphai < 1 + \alphai^2$ and $1+\alphai \geq 1$. 
Telescope the above inequality to obtain
\begin{align}
    \be^{(t), s}
    \leq&
    \bGI^s \be^{(t), 0} + \sum_{k=1}^s \bGI^{s-k} \bb^{(t), k-1} .
    \label{eq:sum_of_inner_consensus_5}
\end{align}
Thus, the sum of the consensus errors can be bounded by 
\begin{align}
    \sum_{s=0}^{S-1} \be^{(t), s}
  &  \leq \be^{(t), 0} + \sum_{s=1}^{S-1} \Big( \bGI^s \be^{(t), 0} + \sum_{k=1}^s \bGI^{s-k} \bb_i^{(t), k-1} \Big) \notag \\
  &  = \sum_{s=0}^{S-1} \bGI^s \be^{(t), 0}
    + \sum_{s=1}^{S-1} \sum_{k=1}^s \bGI^{s-k} \bb^{(t), k-1} \notag \\
 &  \overset{\text{(i)}}{=} \sum_{s=0}^{S-1} \bGI^s \be^{(t), 0}
    + \sum_{k=1}^{S-1} \sum_{s=0}^{S-1-k} \bGI^{s} \bb^{(t), k-1} \notag \\
    & \overset{\text{(ii)}}{\leq}  \sum_{s=0}^{S-1} \bGI^s \be^{(t), 0}
    + \sum_{k=1}^{S-1} \sum_{s=0}^{S-1} \bGI^{s} \bb^{(t), k-1} \notag \\
    & \overset{\text{(iii)}}{\leq} \sum_{s=0}^{\infty} \bGI^s \Big( \be^{(t), 0}
    + \sum_{s=0}^{S-1} \bb^{(t), s} \Big)     \label{eq:sum_of_inner_consensus_3}
\end{align}
where (i) follows by changing the order of summation, (ii) and (iii) follows from the nonnegativity of $\bGI$ and $\bb^{(t), s}$ respectively.
To continue, we begin with the following claim about $\bGI$ which will be proved momentarily.

\begin{claim} \label{claim:spectral_GI}
Under the choice of $\eta$ in Theorem~\ref{theorem:network_sarah_non_convex}, the eigenvalues of $\bGI$ are in $(-1,1)$, and the Neumann series converges, 
\begin{align}  
 \sum_{s=0}^\infty \bGI^s   =  (\bI_2 - \bGI)^{-1}  &  \leq 
    \begin{bmatrix}
        \frac{2}{1 - \alphai}
        & \frac{4 \alphai^2 \eta^2 L^2}{(1 - \alphai)^3} \\[0.5em]
        \frac{32\alphai^2}{(1 - \alphai)^3 p^2}
        & \frac{2}{1 - \alphai}
    \end{bmatrix} .
    \label{eq:sum_of_inner_gradients_9}
 \end{align}
\end{claim}

Let $\bm{\varsigma}_{\mathsf{in}}^{\top} = \begin{bmatrix}
8 \left( \frac{4 S}{pnb} + 1\right) & \frac{8 S \alphai^2 \eta^2 L^2}{p n b}
\end{bmatrix}$,
in view of \Cref{claim:spectral_GI},
the summation of consensus erros in \eqref{eq:lemma:sum_of_inner_loop_errors_1} can be bounded as
\begin{align*}
    &~~~~ \Big( \frac{4S}{npb} + 1 \Big) \cdot 8 L^2 \sum_{s=0}^{S-1} \E \big\| \bu^{(t), s} - \onet\bbu^{(t), s} \big\|_2^2
    + \frac{8 S \alphai^2 \eta^2 L^2}{npb} \sum_{s=0}^{S-1} \E \big\| \bv^{(t), s} - \onet\bbv^{(t), s} \big\|_2^2 \\
    &=\bm{\varsigma}_{\mathsf{in}}^{\top} \sum_{s=0}^{S-1} \be^{(t), s} \\
    &\leq \bm{\varsigma}_{\mathsf{in}}^{\top} \Big( \sum_{s=0}^\infty \bGI^s \Big) \Big( \be^{(t), 0}
    + \sum_{k=0}^{S-1} \bb^{(t), k} \Big) \\
    & \leq \bm{\varsigma}_{\mathsf{in}}^{\top}\big( \bI_2 - \bGI \big)^{-1}  \Big( \be^{(t), 0}
    + \sum_{s=0}^{S-1} \bb^{(t), s} \Big) ,
\end{align*}
and
\begin{align*}
  \bm{\varsigma}_{\mathsf{in}}^{\top}\big( \bI_2 - \bGI \big)^{-1}    &
  \leq
      \begin{bmatrix}
        \frac{16}{1 - \alphai} \Big( \frac{4S}{npb} + 1 \Big) + \frac{32 \alphai^2}{(1 - \alphai)^3 p^2} \cdot \frac{8 S \alphai^2 \eta^2 L^2}{p n b} \,
        & \, \frac{32 \alphai^2 \eta^2 L^2}{(1 - \alphai)^3} \cdot \Big( \frac{4S}{npb} + 1 \Big) + \frac{2}{1 - \alphai} \cdot \frac{8 S \alphai^2 \eta^2 L^2}{npb}
      \end{bmatrix} \\
    & \leq
    \begin{bmatrix}
        \frac{16}{1 - \alphai} \Big( \frac{4S}{npb} + 1 \Big) + \frac{3 \alphai^2}{(1 - \alphai)p^2} \,
        &\, \frac{128 \alphai^2 \eta^2 L^2}{(1 - \alphai)^3} \cdot \Big( \frac{S}{npb} + 1 \Big) + \frac{16 \alphai^2 (1 - \alphai)}{100}
    \end{bmatrix} \\
 &   \leq
    \begin{bmatrix}
        \frac{64}{1 - \alphai}  \Big( \frac{S}{npb} + 1 \Big) \,
        &\,  2\alphai^2
    \end{bmatrix},
\end{align*}
where we use \eqref{eq:step_size_condition},
$\eta L \leq \frac{(1 - \alphai)^3(1 - \alphao)}{10 \big(1 + \alphai \alphao \sqrt{npb} \big) \big( \sqrt{S/(npb)} + 1 \big)} \leq \frac{(1 - \alphai)^3(1 - \alphao)}{10 \big(\sqrt{S/(pnb)} + 1 \big) }$,
to prove the last two inequalities.

Therefore,
\eqref{eq:lemma:sum_of_inner_loop_errors_1}
can be bounded as
\begin{align*}
    2 L^2 \sum_{s=0}^{S-1} \E \big\| & \bu^{(t), s} - \onet\bbu^{(t), s} \big\|_2^2
    + 2n \sum_{s=0}^{S-1} \E \big\| \nabla f(\bbu^{(t), s}) - \bbv^{(t), s} \big\|^2_2 \notag \\
    &\leq \bm{\varsigma}_{\mathsf{in}}^{\top}\big( \bI_2 - \bGI \big)^{-1} \Big( \be^{(t), 0}
    + \sum_{s=0}^{S-1} \bb^{(t), s} \Big) 
    + \frac{4 S \eta^2 L^2}{pb} \sum_{s=0}^{S-1} \E \big\| \bbv^{(t), s} \big\|_2^2 \\
    & \leq \frac{64 L^2}{1 - \alphai} \cdot \Big( \frac{S}{npb} + 1 \Big) \E \big\| \bx^{(t)} - \onet\bbx^{(t)} \big\|_2^2 + 2 \alphai^2 \E \big\| \bs^{(t)} - \onet\bbs^{(t)} \big\|_2^2 \\
    &~~~~~~~~ + \Big(\frac{4\alphai^4 \eta^2 L^2}{(1 - \alphai)^2 p^2} + \frac{4 S \eta^2 L^2}{n p b} \Big) \cdot n \sum_{s=1}^{S} \E \big\| \bbv^{(t), s-1} \big\|^2_2 \\
    & < \frac{64 L^2}{1 - \alphai} \cdot \Big( \frac{S}{npb} + 1 \Big) \E \big\| \bx^{(t)} - \onet\bbx^{(t)} \big\|_2^2 + 2 \alphai^2 \E \big\| \bs^{(t)} - \onet\bbs^{(t)} \big\|_2^2
    + \frac{2n}{25} \sum_{s=1}^{S} \E \big\| \bbv^{(t), s-1} \big\|^2_2,
\end{align*}
where the last inequality is proved by incorporating \eqref{eq:step_size_condition} as
$\frac{4\alphai^4 \eta^2 L^2}{(1 - \alphai)^2 p^2}
\leq \frac{4\alphai^2 \eta^2 L^2}{(1 - \alphai)^2}
< \frac{4\alphai^2}{(1 - \alphai)^2} \cdot
\frac{(1 - \alphai)^6}{100}
\leq \frac{1}{25}$
and
$
\frac{4 S \eta^2 L^2}{n p b}
\leq
\frac{S}{n p b} \cdot
\frac{4}{100 \big( \sqrt{S/(npb)} + 1 \big)^2}
< \frac{1}{25}
$.

\begin{proof}[Proof of Claim~\ref{claim:spectral_GI}]
By the definition of $\bGI$ in \eqref{eq:def_GI}, the characteristic polynomial of $\bGI$ is
\begin{align*}
    f(\lambda)
   & = (\alphai - \lambda) \Big( \alphai + \frac{4 \alphai^4 \eta^2 L^2}{(1 - \alphai) p^2} - \lambda \Big) - \frac{32 \alphai^4 \eta^2 L^2}{(1 - \alphai)^2 p^2} .
\end{align*}

By \eqref{eq:step_size_condition},
$\eta L
\leq \frac{(1 - \alphai)^3(1 - \alphao)}{10 \big(1 + \alphai \alphao \sqrt{npb} \big) \big( \sqrt{S/(npb)} + 1 \big)}
\leq \frac{(1 - \alphai)^3}{10}$ and $\alphai \leq p$,
we have $
\frac{32 \alphai^4 \eta^2 L^2}{(1 - \alphai)^2 p^2}
\leq \frac{32 \alphai^2 \eta^2 L^2}{(1 - \alphai)^2}
\leq \frac{32}{100} \alphai^2 (1 - \alphai)^4
< 1$,
so that $f(-1) \geq 1-\frac{32 \alphai^4 \eta^2 L^2}{(1 - \alphai)^2}   > 0$,
and
\begin{align*}
    f(1)
    &= (1 - \alphai)^2 - \frac{4 \alphai^4 \eta^2 L^2}{p^2} - \frac{32 \alphai^4 \eta^2 L^2}{(1 - \alphai)^2 p^2} \\
   & \geq (1 - \alphai)^2 - \frac{36 \alphai^4 \eta^2 L^2}{(1 - \alphai)^2 p^2} \\
   & > (1 - \alphai)^2 - \frac{36}{100}(1 - \alphai)^4
   > 0.
\end{align*}
Because $f(\alphai) \leq 0$,
all eigenvalues of $\bGI$ are in $(-1, 1)$,
then the Neumann series converges, yielding
\begin{align*}
    \sum_{s=0}^\infty \bGI^s
    & = 
    (\bI_2 - \bGI)^{-1}   \\
    & =
    \frac{1 - \alphai}{(1 - \alphai)^4 p^2 - 4 \big((1 - \alphai)^2 + 8 \big) \alphai^4 \eta^2 L^2}
    \begin{bmatrix}
        (1 - \alphai)^2 p^2 - 4 \alphai^4 \eta^2 L^2
        & 2 \alphai^2 \eta^2 L^2 p^2 \\[0.5em] 
        16 \alphai^2
        & (1 - \alphai)^2 p^2
    \end{bmatrix}   \\
    & \leq
    \frac{1 - \alphai}{(1 - \alphai)^4 p^2 - 4 \big((1 - \alphai)^2 + 8 \big) \alphai^4 \eta^2 L^2}
    \begin{bmatrix}
        (1 - \alphai)^2 p^2
        & 2 \alphai^2 \eta^2 L^2 p^2 \\[0.5em] 
        16 \alphai^2
        & (1 - \alphai)^2 p^2
    \end{bmatrix}   \\
    & \overset{\text{(i)}}{\leq}
    \frac{1 - \alphai}{(1 - \alphai)^4 p^2 - 36 \alphai^4 \eta^2 L^2}
    \begin{bmatrix}
        (1 - \alphai)^2 p^2
        & 2 \alphai^2 \eta^2 L^2 p^2 \\[0.5em] 
        16 \alphai^2
        & (1 - \alphai)^2 p^2
    \end{bmatrix}   \\
    & \overset{\text{(ii)}}{\leq}
    \begin{bmatrix}
        \frac{2}{1 - \alphai}
        & \frac{4 \alphai^2 \eta^2 L^2}{(1 - \alphai)^3} \\[0.5em] 
        \frac{32\alphai^2}{(1 - \alphai)^3 p^2}
        & \frac{2}{1 - \alphai}
    \end{bmatrix} ,
\end{align*}
where (i) and (ii) follow the fact $(1 - \alphai)^2 \leq 1$,
and $(1 - \alphai)^4 p^2 - 36 \alphai^4 \eta^2 L^2
\geq (1 - \alphai)^4 p^2 - \frac{36}{100} \alphai^4 (1 - \alphai)^6
\geq (1 - \alphai)^4 p^2 - \frac{36}{100} \alphai^2 (1 - \alphai)^6 p^2
> \frac{1}{2}(1 - \alphai)^4 p^2$
due to \eqref{eq:step_size_condition}.

\end{proof}

\section{Proof of Lemma \ref{lemma:sum_of_outer_error}}
\label{sub:proof_of_lemma:sum_of_outer_vars}

This section proves \Cref{lemma:sum_of_outer_error}.
In the following subsections,
\Cref{sub:sum_of_outer_consensus_estimation_error,sub:sum_of_outer_gradient_estimation_error} derive induction inequalities for the consensus errors
and \Cref{sub:linear_system} creates a linear system of consensus errors to compute the summation.

\subsection{Sum of outer loop variable consensus errors}
\label{sub:sum_of_outer_consensus_estimation_error}

The variable consensus error can be bounded deterministically as following,
\begin{align}
    &~~~~ \big\| \x^{(t)} -  \onet\bbx^{(t)}  \big\|_2^2 \notag \\
    &= \Big\| \Big(\bI_{nd} - \mean \Big) \x^{(t)} \Big\|_2^2 \notag \\
    & \overset{\text{(i)}}{=} \Big\| \Big(\bI_{nd} - \mean \Big) \bu^{(t-1), S} \Big\|_2^2 \notag \\
    &\overset{\text{(ii)}}{=} \Big\| \Big(\bI_{nd} - \mean \Big) \WKI \Big( \bu^{(t-1), S-1} - \eta \bv^{(t-1), S-1} \Big) \Big\|_2^2 \notag \\
  &  \leq \alphai^2 \Big\| \Big(\bI_{nd} - \mean \Big) \Big( \bu^{(t-1), S-1} - \eta \bv^{(t-1), S-1} \Big) \Big\|_2^2 \notag \\
    & \leq \frac{2 \alphai^2}{1 + \alphai^2} \big\| \bu^{(t-1), S-1} - \onet\bbu^{(t-1), S-1} \big\|_2^2 \notag
    + \frac{2 \alphai^2 \eta^2}{1 - \alphai^2} \big\| \bv^{(t-1), S-1} - \onet\bbv^{(t-1), S-1} \big\|_2^2 \notag ,
\end{align}
where (i) uses $\x^{(t)} = \bu^{(t-1), S}$, (ii) uses the update rule \eqref{eq:inner_loop_var}, and the last two inequalities follow from similar reasonings as \eqref{eq:sum_of_inner_consensus_1}.
Apply the same reasoning to $\frac{2 \alphai^2}{1 + \alphai^2} \big\| \bu^{(t-1), S-1} - \onet\bbu^{(t-1), S-1} \big\|_2^2$
and use $\frac{2\alphai^2}{1 + \alphai^2} \leq 1$,
we can prove
\begin{align}
 \big\| \x^{(t)} - \onet\bbx^{(t)} \big\|_2^2  
  &  \leq \Big( \frac{2 \alphai^2}{1 + \alphai^2} \Big)^S \big\| \bu^{(t-1), 0} - \onet\bbu^{(t-1), 0} \big\|_2^2
    + \frac{2 \alphai^2 \eta^2}{1 - \alphai^2} \sum_{s=0}^{S-1} \big\| \bv^{(t-1), s} - \onet\bbv^{(t-1), s} \big\|_2^2 \notag \\
   & = \Big( \frac{2 \alphai^2}{1 + \alphai^2} \Big)^S \big\| \x^{(t-1)} - \onet\bbx^{(t-1)} \big\|_2^2
   + \frac{2 \alphai^2 \eta^2}{1 - \alphai^2} \sum_{s=0}^{S-1} \big\| \bv^{(t-1), s} - \onet\bbv^{(t-1), s} \big\|_2^2 , \label{eq:sum_of_outer_error_sum_of_inner_erros}
\end{align}
where the last equality identifies $ \x^{(t-1)} = \bu^{(t-1), 0}$.

Take expectation of the previous inequality,
by \eqref{eq:sum_of_inner_consensus_3},
we can further compute the summation in \eqref{eq:sum_of_outer_error_sum_of_inner_erros} as follows
\begin{align*}
    \sum_{s=0}^{S-1} \E \big\| \bv^{(t-1), s} - \onet\bbv^{(t-1), s} \big\|_2^2
    & \leq \frac{32 \alphai^2 L^2}{(1 - \alphai)^3 p^2} \E \big\| \x^{(t-1)} - \onet\bbx^{(t-1)} \big\|_2^2 \notag \\
    &\qquad + \frac{2}{1 - \alphai} \Big( \E \big\| \bs^{(t-1)} - \onet\bbs^{(t-1)} \big\|_2^2 + \frac{2 \alphai^2 \eta^2 L^2}{(1 - \alphai^2) p^2} \cdot n \sum_{s=0}^{S-1} \E \big\| \bbv^{(t-1), s} \big\| \Big).
\end{align*}

Together with $\bx^{(t)} = \bu^{(t), 0}$ and $\bs^{(t)} = \bv^{(t), 0} $,
\eqref{eq:sum_of_outer_error_sum_of_inner_erros} can be further bounded as
\begin{align}
    \E \big\| \x^{(t)} - \onet\bbx^{(t)} \big\|_2^2
    & \leq \Bigg( \Big( \frac{2 \alphai^2}{1 + \alphai^2} \Big)^S + \frac{2 \alphai^2 \eta^2 L^2}{1 - \alphai^2} \cdot \frac{32 \alphai^2}{(1 - \alphai)^3 p^2} \Bigg) \E \big\| \x^{(t-1)} - \onet\bbx^{(t-1)} \big\|_2^2 \notag \\
    &\qquad + \frac{2 \alphai^2 \eta^2}{1 - \alphai^2} \cdot \frac{2}{1 - \alphai} \Big( \big\| \bs^{(t-1)} - \onet\bbs^{(t-1)} \E \big\|_2^2 + \frac{2 \alphai^2 \eta^2 L^2}{(1 - \alphai^2) p^2} \cdot n \sum_{s=0}^{S-1} \E \big\| \bbv^{(t-1), s} \big\| \Big) \notag \\
    &< \alphai \E \big\| \x^{(t-1)} - \onet\bbx^{(t-1)} \big\|_2^2 \notag \\
    &\qquad+ \frac{4 \alphai^2 \eta^2}{(1 - \alphai)^2} \Big( \E \big\| \bs^{(t-1)} - \onet\bbs^{(t-1)} \big\|_2^2 + \frac{2 \alphai^2 \eta^2 L^2}{(1 - \alphai) p^2} \cdot n \sum_{s=0}^{S-1} \E \big\| \bbv^{(t-1), s} \big\| \Big) .
    \label{eq:sum_of_outer_consensus_estimation_error_2}
\end{align}
The last inequality is obtained by using \eqref{eq:step_size_condition} and the fact that $0 \leq \alphai < 1$ as follows
\begin{align*}
    \Big( \frac{2 \alphai^2}{1 + \alphai^2} \Big)^S + \frac{2 \alphai^2 \eta^2 L^2}{1 - \alphai^2} \cdot \frac{32\alphai^2}{(1 - \alphai)^3 p^2}
    &= \Big( \frac{2 \alphai^2}{1 + \alphai^2} \Big)^S + \frac{\alphai^2 (1 - \alphai)^2}{1 + \alphai} \cdot \frac{64 \alphai^2 \eta^2 L^2}{(1 - \alphai)^6 p^2} \\
    &< \frac{2 \alphai^2}{1 + \alphai^2} + \frac{64}{100} \cdot \frac{\alphai^2 (1 - \alphai)^2}{1 + \alphai} \\
   & \leq \frac{2 \alphai^2}{1 + \alphai^2} + \frac{\alphai(1 - \alphai)^2}{1 + \alphai^2}  = \alphai .
\end{align*}

\subsection{Sum of outer loop gradient estimation consensus errors}
\label{sub:sum_of_outer_gradient_estimation_error}

In view of the update rule for the gradient tracking term \eqref{eq:gradient_tracking} and reorganize terms,
\begin{align}
    \big\| \bs^{(t)} - \onet \bbs^{(t)} \big\|_2^2
    &= \Big\| \Big( \bI_{nd} - \mean \Big) \bs^{(t)} \Big\|_2^2 \notag \\
   & = \Big\| \Big( \bI_{nd} - \mean \Big) \WKO \Big( \bs^{(t-1)} + \nabla F(\x^{(t)}) - \nabla F(\x^{(t-1)}) \Big) \Big\|_2^2 \notag \\
    & \leq \frac{2\alphao^2}{1 + \alphao^2} \big\| \bs^{(t-1)} - \onet\bbs^{(t-1)} \big\|_2^2 \notag \\
    &\qquad + \frac{2\alphao^2}{1 - \alphao^2} \Big\| \Big( \bI_{nd} - \mean \Big) \Big( \nabla F(\x^{(t)}) - \nabla F(\x^{(t-1)}) \Big) \Big\|_2^2 ,
    \label{eq:sum_of_outer_gradient_estimation_error_1}
\end{align}
which follows from similar reasonings as \eqref{eq:sum_of_inner_consensus_1}.
The second term can be further decomposed as
\begin{align}
   &~~~~ \Big\| \Big( \bI_{nd} - \mean \Big) \Big( \nabla F(\x^{(t)}) - \nabla F(\x^{(t-1)}) \Big) \Big\|_2^2 \notag \\
   & \leq \big\|  \nabla F(\x^{(t)}) - \nabla F(\x^{(t-1)}) \big\|_2^2 \notag\\
   & \leq L^2 \big\|
   ( \x^{(t)} - \onet\bbx^{(t)}) - ( \x^{(t-1)} -  \onet\bbx^{(t-1)}) 
   + ( \onet\bbx^{(t)} -  \onet\bbx^{(t-1)}) 
    \big\|_2^2 \notag \\
    & = L^2 \big\|
    ( \x^{(t)} - \onet\bbx^{(t)}) - ( \x^{(t-1)} -  \onet\bbx^{(t-1)})  \big\|_2^2
    + nL^2 \big\| \bbx^{(t)} -  \bbx^{(t-1)} \big\|_2^2 \notag \\
    &  \leq 2 L^2 \big\| \x^{(t)} - \onet\bbx^{(t)} \big\|_2^2
    + 2 L^2 \big\| \x^{(t-1)} - \onet\bbx^{(t-1)} \big\|_2^2
    + S \eta^2 L^2 \cdot n \sum_{s=0}^{S-1} \big\| \bbv^{(t-1), s} \big\|_2^2 ,
    \label{eq:sum_of_outer_gradient_estimation_error_2}
\end{align}
where the last line follows from the update rule \eqref{eq:inner_loop_var} by identifying $\bbx^{(t)} - \bbx^{(t-1)} =\eta \sum_{s=0}^{S-1}\bbv^{(t-1), s}$ and Cauchy-Schwartz inequality.

With \eqref{eq:sum_of_outer_gradient_estimation_error_2},
\eqref{eq:sum_of_outer_gradient_estimation_error_1} can be further bounded as follows
\begin{align}
    \big\| \bs^{(t)} - \onet \bbs^{(t)} \big\|_2^2
    & \leq \frac{2\alphao^2}{1 + \alphao^2} \big\| \bs^{(t-1)} - \onet\bbs^{(t-1)} \big\|_2^2
    + \frac{2\alphao^2}{1 - \alphao^2} \Big(
    2 L^2 \big\| \x^{(t)} - \onet\bbx^{(t)} \big\|_2^2 \notag \\
    & \qquad + 2 L^2 \big\| \x^{(t-1)} - \onet\bbx^{(t-1)} \big\|_2^2
    + S \eta^2 L^2 \cdot n \sum_{s=0}^{S-1} \big\| \bbv^{(t-1), s} \big\|_2^2
    \Big) \notag \\
    & \leq \alphao \big\| \bs^{(t-1)} - \onet\bbs^{(t-1)} \big\|_2^2
    + \frac{2\alphao^2}{1 - \alphao} \Big(
    2 L^2 \big\| \x^{(t)} - \onet\bbx^{(t)} \big\|_2^2 \notag \\
    & \qquad + 2 L^2 \big\| \x^{(t-1)} - \onet\bbx^{(t-1)} \big\|_2^2
    + S \eta^2 L^2 \cdot n \sum_{s=0}^{S-1} \big\| \bbv^{(t-1), s} \big\|_2^2
    \Big). \label{eq:fadsfasdfasdfasd}
\end{align}

Combine with \eqref{eq:sum_of_outer_consensus_estimation_error_2},
after taking expectations,
\eqref{eq:fadsfasdfasdfasd} can be further bounded as
\begin{align}
    \E \big\| \bs^{(t)} - \onet \bbs^{(t)} \big\|_2^2
    &< \alphao \E \big\| \bs^{(t-1)} - \onet\bbs^{(t-1)} \big\|_2^2
    + \frac{4 \alphao^2 L^2}{1 - \alphao} \E \big\| \x^{(t-1)} - \onet\bbx^{(t-1)} \big\|_2^2 \notag \\
    & \qquad + \frac{2 \alphao^2 S \eta^2 L^2}{1 - \alphao} \cdot n \sum_{s=0}^{S-1} \E \big\| \bbv^{(t-1), s} \big\|_2^2
    + \frac{4 \alphao^2 L^2}{1 - \alphao} \Bigg( \alphai \E \big\| \x^{(t-1)} - \onet\bbx^{(t-1)} \big\|_2^2 \notag \\
    &\qquad+ \frac{4 \alphai^2 \eta^2}{(1 - \alphai)^2} \Big( \E \big\| \bs^{(t-1)} - \onet\bbs^{(t-1)} \big\|_2^2 + \frac{2 \alphai^2 \eta^2 L^2}{(1 - \alphai) p^2} \cdot n \sum_{s=0}^{S-1} \E \big\| \bbv^{(t-1), s} \big\|_2^2 \Big)  \Bigg) \notag \\
    & = \Big( \alphao + \frac{4 \alphao^2 L^2}{1 - \alphao} \cdot \frac{4 \alphai^2 \eta^2}{(1 - \alphai)^2} \Big) \E \big\| \bs^{(t-1)} - \onet\bbs^{(t-1)} \big\|_2^2 \notag \\
    &\qquad + \frac{4 \alphao^2 L^2}{1 - \alphao} ( 1 + \alphai) \E \big\| \x^{(t-1)} - \onet\bbx^{(t-1)} \big\|_2^2 \notag \\
    & \qquad + \Big( \frac{2 \alphao^2 S \eta^2 L^2}{1 - \alphao}
    + \frac{4 \alphao^2 L^2}{1 - \alphao}
\cdot \frac{4 \alphai^2 \eta^2}{(1 - \alphai)^2} \cdot \frac{2  \alphai^2 \eta^2 L^2}{(1 - \alphai) p^2} \Big) \cdot n \sum_{s=0}^{S-1} \E \big\| \bbv^{(t-1), s} \big\|_2^2 \notag \\
    &\overset{\text{(i)}}{<} \Big( \alphao + \frac{4 \alphao^2 L^2}{1 - \alphao} \cdot \frac{4 \alphai^2 \eta^2}{(1 - \alphai)^2} \Big) \E \big\| \bs^{(t-1)} - \onet\bbs^{(t-1)} \big\|_2^2 \notag \\
    &\qquad + \frac{4 \alphao^2 L^2}{1 - \alphao} ( 1 + \alphai) \E \big\| \x^{(t-1)} - \onet\bbx^{(t-1)} \big\|_2^2 \notag \\
    & \qquad + \frac{3 \alphao^2 S \eta^2 L^2}{1 - \alphao}
     \cdot n \sum_{s=0}^{S-1} \E \big\| \bbv^{(t-1), s} \big\|_2^2,
    \label{eq:sum_of_outer_gradient_estimation_error_3}
\end{align}
where and (i) is obtained by applying the condition in \eqref{eq:step_size_condition} as follows
\begin{align*}
    \frac{4 \alphao^2 L^2}{1 - \alphao}
    \cdot \frac{4 \alphai^2 \eta^2}{(1 - \alphai)^2} \cdot \frac{2  \alphai^2 \eta^2 L^2}{(1 - \alphai) p^2}
    &= \frac{ \alphao^2 \eta^2  L^2}{1 - \alphao}
    \cdot \frac{32  \alphai^4 \eta^2 L^2}{(1 - \alphai)^3 p^2}  \\
    &\leq \frac{ \alphao^2 S \eta^2  L^2}{1 - \alphao}
    \cdot \frac{32\alphai^2 (1 - \alphai)^6}{100(1 - \alphai)^3} \\
    & \leq \frac{\alphao^2 S \eta^2 L^2}{1 - \alphao},
\end{align*}
where the inequalities are obtained by using $S \geq 1$ and $0 \leq \alphai < 1$.

\subsection{Linear system}%
\label{sub:linear_system}

Defining $\be^{(t)}:= \be^{(t),0}
= \begin{bmatrix}
    L^2 \E \big\| \x^{(t)} - \onet\bbx^{(t)} \big\|_2^2 \\[0.3em]
    \E \big\| \bs^{(t)} - \onet\bbs^{(t)} \big\|_2^2 \\
\end{bmatrix}$
and
$\bb'^{(t)} = \begin{bmatrix}
    \frac{8 \alphai^4 \eta^4 L^4}{(1 - \alphai)^3 p^2} \cdot n \sum_{s=0}^{S-1} \E \big\| \bbv^{(t), s} \big\|_2^2 \\[0.3em] 
    \frac{3 \alphao^2 S \eta^2 L^2}{ 1 - \alphao} \cdot n \sum_{s=0}^{S-1} \E \big\| \bbv^{(t), s} \big\|_2^2
\end{bmatrix}$, 
we construct a linear system by putting together \eqref{eq:sum_of_outer_consensus_estimation_error_2} and \eqref{eq:sum_of_outer_gradient_estimation_error_3} as
\begin{align}
    \be^{(t)}
   & \leq
  \underbrace{  \begin{bmatrix}
        \alphai & \frac{4 \alphai^2 \eta^2 L^2}{(1 - \alphai)^2} \\[0.5em]
        \frac{4 \alphao^2}{1 - \alphao} (1 + \alphai ) & \alphao + \frac{4 \alphao^2}{1 - \alphao} \cdot \frac{4 \alphai^2 \eta^2 L^2}{(1 - \alphai)^2}
    \end{bmatrix}}_{=: \bGO}
    \be^{(t-1)}
    +
    \bb'^{(t-1)}   = \bGO \be^{(t-1)}
    +
    \bb'^{(t-1)}  .
    \label{eq:linear_system_1}
\end{align}

Then, following the same argument as \eqref{eq:sum_of_inner_consensus_3}, we obtain
\begin{align}
    \sum_{t=0}^T \be^{(t)}
   & \leq \sum_{t=0}^\infty \bGO^t \Big( \be^{(0)}
    + \sum_{t=0}^{T-1} \bb'^{(t)} \Big) . \label{eq:sum_of_outer_consensus_sum_error}
\end{align}

Before continuing, we state the following claim about $\bGO$ which will be proven momentarily.
\begin{claim} \label{claim:spectral_GO}
Under the choice of $\eta$ in Theorem~\ref{theorem:network_sarah_non_convex}, the eigenvalues of $\bGO$ are in $(-1,1)$, and the Neumann series converges, 
\begin{align*}  
    \sum_{t=0}^\infty \bGO^t
    = (\bI_2 - \bGO)^{-1}
    & \leq 
    \begin{bmatrix}
        \frac{2}{1 - \alphai}
        & \frac{8 \alphai^2 \eta^2 L^2}{(1 - \alphai)^3 (1 - \alphao)} \\[0.5em]
        \frac{16 \alphao^2}{(1 - \alphai) (1 - \alphao)^2}
        & \frac{2}{1 - \alphao} \\
    \end{bmatrix}
     .
\end{align*}
\end{claim}

With Claim~\ref{claim:spectral_GO} in hand,
and the fact that $\be^{(0)} = \bm0$,
we can bound the summation of outer loop consensus errors by
\begin{align}
    &~~~~ \frac{64 L^2}{1 - \alphai} \cdot \Big( \frac{S}{npb} + 1 \Big) \sum_{t=0}^T \E \big\| \bx^{(t)} - \onet\bbx^{(t)} \big\|_2^2 + \frac{2 \alphai^2}{1 - \alphai} \sum_{t=0}^T \E \big\| \bs^{(t)} - \onet\bbs^{(t)} \big\|_2^2 \notag \\
    &= \bm{\varsigma}_{\mathsf{out}}^\top \sum_{t=1}^T \be^{(t)} \notag \\
    &\leq \bm{\varsigma}_{\mathsf{out}}^\top  (\bI_2 - \bGO)^{-1} \Big( \be^{(0)}
    + \sum_{t=0}^{T-1} \bb'^{(t)} \Big) \notag \\
    &= \bm{\varsigma}_{\mathsf{out}}^\top  (\bI_2 - \bGO)^{-1}
    \sum_{t=0}^{T-1} \bb'^{(t)} ,
    \label{eq:proof_of_lemma:sum_of_outer_vars_1}
\end{align}
where $\bm{\varsigma}_{\mathsf{out}}^\top =
\begin{bmatrix}
\frac{64 }{1 - \alphai} \cdot \Big( \frac{S}{npb} + 1 \Big) & \frac{2 \alphai^2}{1 - \alphai}
\end{bmatrix}$.

Note that by elementary calculations,
\begin{align*}
    &~~~~ \bm{\varsigma}_{\mathsf{out}}^\top
    (\bI_2 - \bGO)^{-1} \\
    &\leq
    \begin{bmatrix}
        \frac{64}{1 - \alphai}  \Big( \frac{S}{npb} + 1 \Big) & 2 \alphai^2
    \end{bmatrix}
    \begin{bmatrix}
    \frac{2}{1 - \alphai}
    & \frac{8 \alphai^2 \eta^2 L^2}{(1 - \alphai)^3 (1 - \alphao)} \\[0.5em]
    \frac{16 \alphao^2}{(1 - \alphai) (1 - \alphao)^2}
    & \frac{2}{1 - \alphao} \\
    \end{bmatrix} \\
    &=
    \begin{bmatrix}
    \frac{64}{1 - \alphai}  \Big( \frac{S}{npb} + 1 \Big) \cdot \frac{2}{1 - \alphai} + \frac{32 \alphai^2 \alphao^2}{(1 - \alphai) (1 - \alphao)^2}
    & \frac{64}{1 - \alphai}  \Big( \frac{S}{npb} + 1 \Big) \cdot \frac{8 \alphai^2 \eta^2 L^2}{(1 - \alphai)^3 (1 - \alphao)} 
    + \frac{4 \alphai^2}{1 - \alphao}
    \end{bmatrix} \\
    &\overset{\text{(i)}}{<}
    \begin{bmatrix}
        \frac{128}{(1 - \alphai)^2}  \Big( \frac{S}{npb} + 1 \Big) + \frac{32 \alphai^2 \alphao^2}{(1 - \alphai) (1 - \alphao)^2}
    & 6 \alphai^2 + \frac{4 \alphai^2}{1 - \alphao}
    \end{bmatrix} \\
    &\overset{\text{(ii)}}{<}
    \begin{bmatrix}
        \frac{128}{(1 - \alphai)^2}  \Big( \frac{S}{npb} + 1 \Big) + \frac{32 \alphai^2 \alphao^2}{(1 - \alphai) (1 - \alphao)^2}
    & \frac{10 \alphai^2}{1 - \alphao}
    \end{bmatrix},
\end{align*}
where we use \eqref{eq:step_size_condition} to prove (i),
and $1 / (1 - \alphai) \geq 1$ and $1 / (1 - \alphao) \geq 1$ to prove (ii).

Thus, \eqref{eq:proof_of_lemma:sum_of_outer_vars_1} can be bounded using \eqref{eq:step_size_condition} as
\begin{align*}
    &~~~~ \bm{\varsigma}_{\mathsf{out}}^\top
    (\bI_2 - \bGO)^{-1}
    \begin{bmatrix}
    \frac{8 \alphai^4 \eta^4 L^4}{(1 - \alphai)^3 p^2} \\[0.3em] 
    \frac{3 \alphao^2 S \eta^2 L^2}{ 1 - \alphao}
    \end{bmatrix} \\
    &\leq \Bigg( \frac{128}{(1 - \alphai)^2}  \Big( \frac{S}{npb} + 1 \Big) + \frac{32 \alphai^2 \alphao^2}{(1 - \alphai) (1 - \alphao)^2} \Bigg) \frac{8 \alphai^4 \eta^4 L^4}{(1 - \alphai)^3 p^2}
    + \frac{10 \alphai^2}{1 - \alphao} \cdot \frac{3 \alphao^2 S \eta^2 L^2}{ 1 - \alphao} \\
    &= \frac{1024 \alphai^4 \eta^4 L^4}{(1 - \alphai)^5 p^2}  \Big( \frac{S}{npb} + 1 \Big)
    + \frac{256 \alphai^6 \alphao^2 \eta^4 L^4}{(1 - \alphai)^4 (1 - \alphao)^2 p^2}
    + \frac{30 \alphai^2 \alphao^2 npb \cdot S/(npb)}{(1 - \alphao)^2}
    \cdot \eta^2 L^2 \\
    &\leq
    11 \alphai^4 \eta^2 L^2
    + 3 \alphai^6 \alphao^2 \eta^2 L^2
    + \frac{30}{100} \\
    &< \frac{11}{25},
\end{align*}
which concludes the proof.

\begin{proof}[Proof of \Cref{claim:spectral_GO}]
For simplicity,
denote $c = \frac{4 \alphai^2 \eta^2 L^2}{(1 - \alphai)^2}$ and $d = \frac{4 \alphao^2}{1 - \alphao}$. Then
$\bGO$ can be written as
\begin{align*}
    \bGO
    =
    \begin{bmatrix}
        \alphai & c \\
        d (1 + \alphai ) & \alphao + cd
    \end{bmatrix} ,
\end{align*}
whose characteristic polynomial is
\begin{align*}
    f(\lambda)
    &= ( \alphai - \lambda ) (\alphao + cd - \lambda) - ( 1 + \alphai ) cd.
\end{align*}
First, note that $f(1)$ can be bounded by
\begin{align*}
    f(1)
    &= ( \alphai - 1 )(\alphao + cd - 1) - ( 1 + \alphai ) cd \\
    &= (1 - \alphai)(1 - \alphao) -2 cd > 0,
\end{align*}
where the last inequality is due to the choice of $\eta$, namely,
\begin{align*}
    cd
    &= \frac{4 \alphao^2}{1 - \alphao} \cdot \frac{4 \alphai^2 \eta^2 L^2}{(1 - \alphai)^2}      \leq \frac16 (1 - \alphai)(1 - \alphao) .
\end{align*}

Combined with the trivial fact that $f(-1) > 0$ and $f(\alphai) \leq 0$,
all eigenvalues of $\bGO$ are in $(-1, 1)$. Consequently,
the Neumann series converges, leading to
\begin{align*}
     \sum_{t=0}^\infty \bGO^t =   (\bI_2 - \bGO)^{-1}
    & =
    \begin{bmatrix}
    \frac{(1 - \alphai)^2 (1 - \alphao)^2 - 16 \alphai^2 \alphao^2 \eta^2 L^2}{(1 - \alphai)^3 (1 - \alphao)^2 - 32 \alphai^2 \alphao^2 \eta^2 L^2}
    & \frac{4 \alphai^2 (1 - \alphao) \eta^2 L^2}{(1 - \alphai)^3 (1 - \alphao)^2 - 32 \alphai^2 \alphao^2 \eta^2 L^2} \\[0.5em]
    \frac{4 (1 - \alphai)^2 (1 + \alphai) \alphao^2}{(1 - \alphai)^3 (1 - \alphao)^2 - 32 \alphai^2 \alphao^2 \eta^2 L^2}
    & \frac{(1 - \alphai)^3 (1 - \alphao)}{(1 - \alphai)^3 (1 - \alphao)^2 - 32 \alphai^2 \alphao^2 \eta^2 L^2} 
    \end{bmatrix} \\
    &\leq
    \begin{bmatrix}
    \frac{2}{1 - \alphai}
    & \frac{8 \alphai^2 \eta^2 L^2}{(1 - \alphai)^3 (1 - \alphao)} \\[0.5em]
    \frac{16 \alphao^2}{(1 - \alphai) (1 - \alphao)^2}
    & \frac{2}{1 - \alphao} \\
    \end{bmatrix},
\end{align*}
where we use the condition in \eqref{eq:step_size_condition} to prove
$ 32 \alphai^2 \alphao^2 \eta^2 L^2
    \leq \frac{32}{100} (1 - \alphai)^6(1 - \alphao)^2
    < \frac12 (1 - \alphai)^3(1 - \alphao)^2
$
to bound the denominator.

\end{proof}

\end{document}